\documentclass[letterpaper, 10 pt, conference]{ieeeconf}

\IEEEoverridecommandlockouts

\usepackage{mathrsfs}
\usepackage{amsmath}
\usepackage{amssymb}
\usepackage{amsfonts}
\usepackage{amsthm}
\usepackage{comment}
\usepackage{multirow}

\usepackage{enumerate}
\usepackage{cancel}
\usepackage{graphicx}
\usepackage{algorithm}
\usepackage{algpseudocode}
\usepackage[nameinlink,capitalise]{cleveref}
\usepackage{balance}
\usepackage{xcolor}
\usepackage{soul}
\usepackage[normalem]{ulem}

\usepackage{amssymb}

\DeclareMathOperator*{\argmin}{argmin}

\newtheorem{theorem}{Theorem}
\newtheorem{lemma}[]{Lemma}
\newtheorem{corollary}{Corollary}[theorem]
\newtheorem{definition}[]{Definition}

\newtheorem{problem}[]{Problem}
\newtheorem*{assumptions*}{Assumptions}
\newtheorem{assumption}[]{Assumption}

\newtheorem{example}{Example}

\newcounter{sideremark}

\newcommand{\gtrlesseqslant}{%
  \mathrel{\vcenter{\offinterlineskip
    \ialign{%
      \hfil$\mathsurround=0pt ##$\cr
      >\cr
      \noalign{\vskip-0.25ex}
      \leqslant\cr
    }%
  }}%
}

\title{\LARGE \bf Exploiting Trust for Resilient Hypothesis Testing with Malicious Robots}

\author{Matthew Cavorsi*, Orhan Eren Akgün*, Michal Yemini, Andrea Goldsmith, and Stephanie Gil
\thanks{*Co-primary authors}
\thanks{Matthew Cavorsi, Orhan Eren Akgün, and Stephanie Gil are with the School of Engineering and Applied Sciences, Harvard University, Cambridge, MA, USA
        {\tt\small mcavorsi@g.harvard.edu, erenakgun@g.harvard.edu, sgil@seas.harvard.edu}}%
\thanks{Michal Yemini and Andrea Goldsmith are with the Department of Electrical and Computer Engineering, Princeton University, Princeton, NJ, USA
        {\tt\small myemini@princeton.edu, goldsmith@princeton.edu}}%
\thanks{
This work was partially supported by NSF award \#CNS-2147694. M.~Cavorsi, O.~Akgün, and S.~Gil are partially supported by AFOSR award \#FA9550-22-1-0223.  M.~Yemini and A.~J.~Goldsmith are partially supported by  AFOSR award \#002484665.}
}

\begin{document}

\maketitle
\thispagestyle{empty}
\pagestyle{empty}

\begin{abstract}

We develop a resilient binary hypothesis testing framework for decision making in adversarial multi-robot crowdsensing tasks. This framework exploits stochastic trust observations between robots to arrive at tractable, resilient decision making at a centralized Fusion Center (FC) even when i) there exist malicious robots in the network and their number may be larger than the number of legitimate robots, and ii) the FC uses one-shot noisy measurements from all robots. We derive two algorithms to achieve this. The first is the Two Stage Approach (2SA) that estimates the legitimacy of robots based on received trust observations, and provably minimizes the probability of detection error in the worst-case malicious attack. Here, the proportion of malicious robots is known but arbitrary. For the case of an unknown proportion of malicious robots, we develop the Adversarial Generalized Likelihood Ratio Test (A-GLRT) that uses both the reported robot measurements and trust observations to estimate the trustworthiness of robots, their reporting strategy, and the correct hypothesis simultaneously. We exploit special problem structure to show that this approach remains computationally tractable despite several unknown problem parameters. We deploy both algorithms in a hardware experiment where a group of robots conducts crowdsensing of traffic conditions on a mock-up road network similar in spirit to Google Maps, subject to a Sybil attack. We extract the trust observations for each robot from actual communication signals which provide statistical information on the uniqueness of the sender. We show that even when the malicious robots are in the majority, the FC can reduce the probability of detection error to $30.5\%$ and $29\%$ for the 2SA and the A-GLRT respectively.

\end{abstract}

\IEEEpeerreviewmaketitle

\section{Introduction}
We are interested in the problem where robots observe the environment and estimate the presence of an event of interest. Each robot relays their measurement to a \emph{Fusion Center} (FC) that makes an informed binary decision on the occurrence of the event. An unknown subset of the system are malicious robots whose goal is to increase the likelihood that the FC makes a wrong decision~\cite{kailkhura2014asymptotic, ren2018binary, althunibat2016countering, wu2018generalized}. This problem can be cast as an adversarial binary hypothesis testing problem, with relevance to a broad class of robotics tasks that rely on distributed sensing with possibly malicious or untrustworthy robots. For example, robots might perform coordinated coverage to maximize their ability to sense events of interest~\cite{robotTrustSchwager,xu2021novel,song2020care,talay2009task}, share target information for coordinated tracking~\cite{schlotfeldt2018resilient, ramachandran2020resilience,mitra2019resilient,laszka2015resilient}, or merge map information to provide a global understanding of the environment~\cite{blumenkamp2021emergence,mitchell2020gaussian, deng2021byz, wehbe2022probabilistically}. In crowdsensing tasks such as traffic prediction, a server may use GPS data to estimate if a particular roadway is congested or not~\cite{GoogleMaps} (see \cref{fig:map_spoof}). Unfortunately, this process is vulnerable to malicious robots~\cite{kailkhura2014asymptotic, althunibat2016countering}. For example, prior works have shown that a Sybil attack can cause crowdsensing applications like Google Maps to incorrectly perceive traffic conditions, resulting in erroneous reporting of traffic flows~\cite{jeske2013floating, wang2018ghost}.

\begin{figure}[b!]
    \centering
    \includegraphics[scale=0.23]{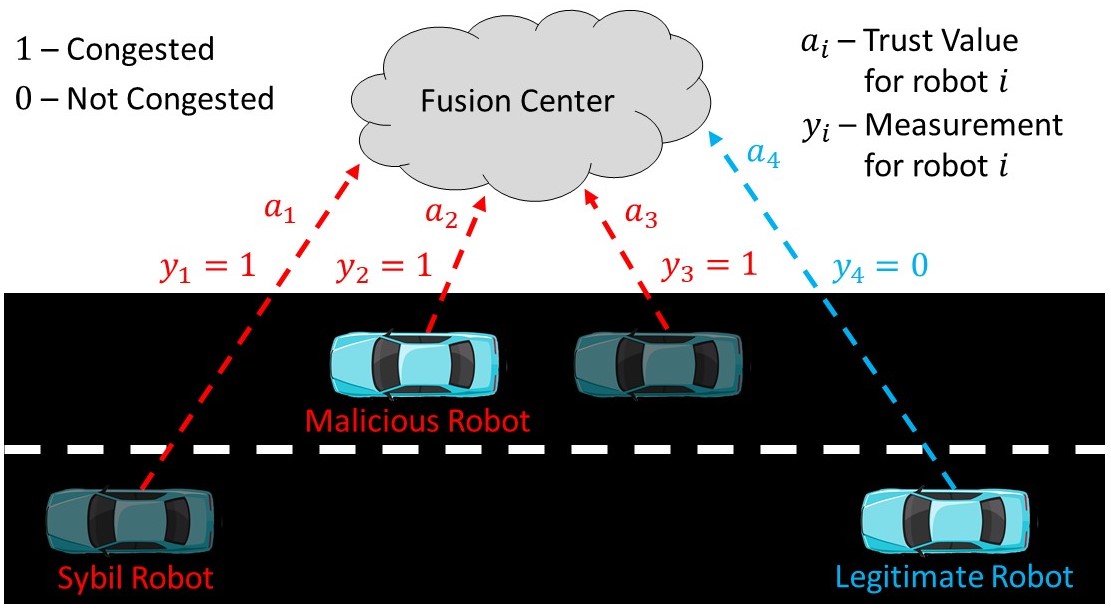}
    \caption{Malicious robots can perform a Sybil Attack to try to force a FC to incorrectly perceive traffic conditions on a road. The FC can aggregate measurements and trust values from robots to accurately estimate the true traffic condition of the road despite the attack.}
    \label{fig:map_spoof}
\end{figure}

The problem of binary adversarial hypothesis testing has been studied within the context of sensor networks~\cite{sandal2020reputation,  marano2008distributed, kailkhura2015distributed}. Many approaches use data, such as a history of measurements and hypothesis outcomes, to assess the trustworthiness of the robots~\cite{chen2008robust,nurellari2017secure, nurellari2016distributed, rawat2010collaborative}. For example, if a robot consistently disagrees with the final decision of the FC, then the FC can flag that robot as potentially adversarial. However, the success of these methods often hinges upon a crucial assumption that \emph{more than half of the network is legitimate.} A growing body of work investigates additionally sensed quantities arising from the \emph{physicality of cyberphysical systems} such as multi-robot networks, to cross-validate and assess the trustworthiness of robots~\cite{trustandRobotsSycara,robotTrustSchwager,spoofResilientCoordinationusingFingerprints,securearray}. This could include using camera feeds, GPS signals, or even the signatures of received wireless communication signals, to acquire additional information regarding \emph{the trustworthiness} of the robots~\cite{AURO,securearray,CrowdVetting}. Importantly, this class of trust observations can often be obtained from a one-shot observation, independent of the transmitted measurement. The work in~\cite{yemini2021characterizing} uses trust observations to recover resilient consensus even in the case where more than half of the network is malicious. In this paper \emph{we wish to derive a framework for adversarial hypothesis testing that exploits stochastic trust observations to arrive at a similar level of resilience; whereby, a FC can conceivably reduce its probability of error, even in the one-shot scenario and where legitimate robots do not hold a majority in the network.}

We derive algorithms for achieving resilient hypothesis testing by exploiting stochastic trust observations between the FC and a group of robots participating in event detection. We derive a framework that exploits one-shot trust observations, hereafter called \emph{trust values}, over each link to arrive at \emph{tractable, closed-form solutions} when the majority of the network may be malicious and the strategy of the malicious robots is unknown -- a challenging and otherwise intractable problem to solve in the general case~\cite{soltanmohammadi2012decentralized}. 

For the case where an upper limit on the proportion of malicious robots is known, we develop the \emph{Two Stage Approach} (2SA). In the first stage this algorithm uses trust values to determine the most likely set of malicious robots, and then applies a Likelihood Ratio Test (LRT) only over trusted robots in the second stage. We show that this approach minimizes the error probability of the estimated hypothesis at the FC for a worst-case attack scenario. For the case where an upper bound on the proportion of malicious robots is unknown, we develop the \emph{Adversarial Generalized Likelihood Ratio Test} (A-GLRT) algorithm which uses both stochastic trust values and event measurements to jointly estimate the trustworthiness of each robot, the strategy of malicious robots, and the hypothesis of the event. Our A-GLRT algorithm is based upon a common approach for decision making with unknown parameters, the Generalized Likelihood Ratio Test (GLRT), which replaces the unknown parameters with their maximum likelihood estimates (MLE)~\cite{kay_2008}. We show that the addition of trust values allows us to decouple the trustworthiness estimation from the strategy of the adversaries, allowing us to calculate the exact MLE of unknown parameters in polynomial time, instead of approximating them as in previous works~\cite{soltanmohammadi2012decentralized,sun2016optimal}. Our simulation results show that the A-GLRT empirically yields a lower probability of error than the 2SA, but at the expense of higher computational cost. 

Finally we conduct a hardware experiment based on crowdsensing traffic conditions using a group of robots under a Sybil Attack. We show that the FC can recover a performance of $30.5\%$ and $29.0\%$ error, for the 2SA and A-GLRT respectively, even in the case where more than half of the robots are malicious.

\section{Problem Formulation}
\label{sec:PF}

We consider a network of $N$ robots, where each robot is indexed by some $i \in \mathcal{N}$ and $\mathcal{N} = \{ 1, \dots, N \}$, that are deployed to sense an environment and determine if an event of interest has occurred. The event of interest is captured by the random variable $\Xi$, where $\Xi=1$ if the event has happened and $\Xi=0$ otherwise. Each robot $i$ uses its sensed information to make a local decision about whether the event has happened or not, captured by the random variable $Y_i$, where its realization $y_i = 1$ if robot $i$ believes the event has happened and $y_i = 0$ otherwise. We denote the true hypothesis by $\mathcal{H}_1$ if $\Xi = 1$ and $\mathcal{H}_0$ if $\Xi = 0$. Each robot forwards its local decision to a centralized fusion center (FC).

We are concerned with the scenario where not all robots are trustworthy, that is, some are \emph{malicious} and may manipulate the data that they send to the FC by flipping their measured bit with the goal of increasing the probability that the FC makes the wrong decision. We denote the set of malicious robots by $\mathcal{M} \subset \mathcal{N}$. The set of robots that are not malicious are termed \emph{legitimate robots}, denoted by $\mathcal{L} \subseteq \mathcal{N}$, where $\mathcal{L}\cup\mathcal{M}=\mathcal{N}$ and $\mathcal{L}\cap\mathcal{M}=\emptyset$. Additionally, we define the true trust vector, $\mathbf{t} \in \{0,1\}^N$, where $t_i = 1$ if $i \in \mathcal{L}$ and $t_i = 0$ if $i \in \mathcal{M}$. We note that the true trust vector is unknown by the FC, but it is defined for analytical purposes. We are interested in estimating this vector.

We assume the following behavioral models for legitimate and malicious robots:
\begin{definition}[Legitimate robot]
A \emph{legitimate robot} $i$ measures the event and sends its measurement $Y_i$ to the FC without altering it. We assume for each legitimate robot $i\in\mathcal{L}$, the measured bit $Y_i$ is subject to noise with the following false alarm and missed detection probabilities
\begin{equation}\label{eq:p_FA_MD_legitimate_sensor}
\begin{aligned}
P_{\text{FA},i} &= \Pr(Y_i=1|\Xi=0,t_i=1) = P_{\text{FA,L}}, \\
P_{\text{MD},i} &= \Pr(Y_i=0|\Xi=1,t_i=1) = P_{\text{MD,L}},
\end{aligned}
\end{equation} 
where $P_{\text{FA,L}} \in (0,0.5)$ and $P_{\text{MD,L}} \in (0,0.5)$ without loss of generality. We assume that all legitimate robots have homogeneous sensing capabilities, i.e., they have the same probability of false alarm and missed detection. Moreover, we assume that the measurement of a legitimate robot is independent of all other robots, and identically distributed given the true hypothesis. Finally, we also assume that $P_{\text{FA,L}}$ and $P_{\text{MD,L}}$ are known by the FC.
\end{definition}
\begin{definition}[Malicious robot]
A robot is said to be a \emph{malicious robot} if it can choose to alter its measurements before sending it to the FC.  We assume that a malicious robot $i\in\mathcal{M}$ can flip its measurement with probability $p_{\text{f}}\in [0,1]$ after making an observation, and that all malicious robots flip their bit with the same probability. Let $p_{\text{FA}, \text{M}}, p_{\text{MD}, \text{M}} \in [0,0.5)$ be the probability of false alarm and missed-detection of a malicious robot before altering the bit. We assume that all malicious robots have the same probability of false alarm and missed detection. The effective probabilities of false alarm and missed-detection of a malicious robot after altering the bit are given as:
\begin{flalign}\label{eq:p_FA_MD_malicious_sensor}
 P_{\text{FA,M}}&= \Pr(Y_i=1|\Xi=0,t_i=0) \\ &= (1-p_{\text{f}})\cdot p_{\text{FA,M}}+p_{\text{f}}\cdot (1-p_{\text{FA,M}}),\nonumber\\
P_{\text{MD,M}} &= \Pr(Y_i=0|\Xi=1,t_i=0) \\ &= (1-p_{\text{f}}) \cdot p_{\text{MD,M}}+p_{\text{f}}(1-p_{\text{MD,M}}).     \nonumber
\end{flalign}
We assume that a measurement coming from a malicious robot is independent of other measurements given the true hypothesis. This implies that malicious robots do not cooperate with each other. Furthermore, we assume that $p_{\text{FA}, \text{M}}$, $p_{\text{MD}, \text{M}}$, and the strategy of the malicious robots, which is the flipping probability $p_{\text{f}}$, are not known by the FC. This implies that the FC does not know $P_{\text{FA}, \text{M}}$ and $P_{\text{MD}, \text{M}}$ either.
\end{definition}
We use a common assumption in the literature which is that the measurements coming from malicious robots are i.i.d (see \cite{ren2018binary, wu2018generalized, marano2008distributed, kailkhura2015distributed}). In addition to the measurements $Y_i$, we assume that each $Y_i$ is tagged with a \emph{trust value} $\alpha_i \in \mathbb{R}$. Specifically, we consider the class of problems where the FC can leverage the cyber-physical nature of the network to extract an estimation of trust about each communicating robot.
\begin{definition}[Trust Value $\alpha_i$]\label{def:trust_val}
A \emph{trust value} $\alpha_i$ is a stochastic variable  that captures information about the true legitimacy of a robot $i$. We denote the set of all possible trust values (aka sample space) by $\mathcal{A}$ and denote a realization for robot $i$ by $a_i$.
\end{definition}

\begin{assumption}\label{assumption:trust_val}
We assume that the set $\mathcal{A}$ is finite and that the trust value distributions are homogeneous across all the legitimate robots $i\in\mathcal{L}$. To this end, we denote the probability mass function of the trust values of robots by $p_{\alpha}(a|t)$.
We assume the probability mass functions are known or can be estimated by the FC.\footnote{\scriptsize{ Example of a trust value $\alpha_i$: One example of such trust values comes from the works in \cite{yemini2021characterizing, AURO, CrowdVetting}. In these works, the trust values $\alpha_i \in [0,1]$ are stochastic and are determined from physical properties of wireless transmissions. We use these trust values in our hardware experiment in \cref{sec:Results} where we discretize the sample space by letting $\mathcal{A} = \{0,1\}$ and find the probability mass functions to be $p_{\alpha}(a_i=1|t_i=1) = 0.8350$ and $p_{\alpha}(a_i=1|t_i=0) = 0.1691$. Other examples of observations can be found in \cite{trustandRobotsSycara, cheng2021general, peng2012agenttms}.}} We assume that the trust values are i.i.d given the true legitimacy of the robot. 
Moreover, the trust values are assumed to be independent of the measurements, $Y_i$, and the true hypothesis. Finally, to omit trivial or noninformative cases, we assume that $p_{\alpha}(a| t = 0)\cdot p_{\alpha}(a| t = 1)\notin\{0,1\}$  for all $a\in\mathcal{A}$.
\end{assumption}

We do not impose any restrictions over the conditional probability distributions $p_{\alpha}(a | t = 1)$ and $p_{\alpha}(a | t = 0)$. However, for the trust values to be meaningful they should have different probability mass functions, i.e., $p_{\alpha}(a | t = 1) \neq p_{\alpha}(a | t = 0)$.
How distinguishable the two probability mass functions are is termed the \emph{quality} of the trust value, where a better quality corresponds to a larger distinction between the distributions $p_{\alpha}(a| t = 1)$ and $p_{\alpha}(a | t = 0)$. Based on these definitions, we provide the objective of the FC.

\subsection{The objective of the FC}

Denote the vector of all measurements with $\boldsymbol{Y}=(Y_{1},\ldots,Y_{N})$ and its realization $\boldsymbol{y}=(y_1,\ldots,y_N)$, and the vector of stochastic trust values by $\boldsymbol{\alpha}=(\alpha_{1},\ldots,\alpha_{N})$ and its realization by $\boldsymbol{a}=(a_{1},\ldots,a_{N})$. Let $\mathcal{D}_0$ and $\mathcal{D}_1$ be the decision regions at the FC. That is, $(\boldsymbol{a},\boldsymbol{y})\in\mathcal{D}_0$ if the FC chooses hypothesis $\mathcal{H}_0$ whenever it measures the  pair $(\boldsymbol{a},\boldsymbol{y})$. Similarly $(\boldsymbol{a},\boldsymbol{y})\in\mathcal{D}_1$ if the FC chooses hypothesis $\mathcal{H}_1$ whenever it measures the pair $(\boldsymbol{a},\boldsymbol{y})$. To simplify our notations we denote $\mathcal{D}:=\{\mathcal{D}_0,\mathcal{D}_1\}$.

Denote by $P_{\text{FA}}$ and $P_{\text{MD}}$ the false alarm and missed detection probabilities of the decision rule used by the FC, that is 
\begin{flalign}
&P_{\text{FA}}(\mathcal{D},\boldsymbol{t},P_{\text{FA,M}}) \nonumber\\
&\qquad= \sum_{(\boldsymbol{a},\boldsymbol{y})\in\mathcal{D}_1}\Pr(\boldsymbol{\alpha}=\boldsymbol{a},\boldsymbol{Y}=\boldsymbol{y}|\mathcal{H}_0,\boldsymbol{t},P_{\text{FA,M}}),\\
&P_{\text{MD}}(\mathcal{D},\boldsymbol{t},P_{\text{MD,M}}) \nonumber\\
&\qquad= \sum_{(\boldsymbol{a},\boldsymbol{y})\in\mathcal{D}_0}\Pr(\boldsymbol{\alpha}=\boldsymbol{a},\boldsymbol{Y}=\boldsymbol{y}|\mathcal{H}_1,\boldsymbol{t},P_{\text{MD,M}}).
\end{flalign}
Note that the false alarm and missed detection probabilities are affected by the strategy of the malicious robots, i.e., $P_{\text{FA,M}}$ and $P_{\text{MD,M}}$.

If the FC knows the true trust vector, i.e., the vector $\boldsymbol{t}$, and the probabilities $P_{\text{FA,M}}$ and $P_{\text{MD,M}}$, it could optimize the decision regions $\mathcal{D}_0$ and  $\mathcal{D}_1$ to minimize the expected error probability:
\begin{equation}
\begin{aligned}\label{eq:straightforward_obj_opt}
&P_{\text{e}}(\mathcal{D},\boldsymbol{t},P_{\text{FA,M}},P_{\text{MD,M}})= \\
&  \Pr(\Xi=0)P_{\text{FA}}(\mathcal{D},\boldsymbol{t},P_{\text{FA,M}})+\Pr(\Xi=1)P_{\text{MD}}(\mathcal{D},\boldsymbol{t},P_{\text{MD,M}}).
\end{aligned}
\end{equation}
In this case, the vector of trust values $\boldsymbol{\alpha}$ would not affect the optimal decision rule, and it would only depend on the vector of measurements $\boldsymbol{Y}$. 

However, there are two main obstacles to the optimization of the probability of error \eqref{eq:straightforward_obj_opt}, namely:
\begin{enumerate}
    \item The FC does not know the identity of the malicious robots, and thus it does not know the correct vector $\boldsymbol{t}$. Therefore, the FC needs to estimate the true trust vector, where the estimated trust vector is denoted by $\hat{\boldsymbol{t}}$.
    \item The FC does not know how the malicious robots alter their measurements before sending them. In our setup, this means that the FC does not know the values $P_{\text{FA,M}}$ and $P_{\text{MD,M}}$. Therefore the FC needs to estimate $P_{\text{FA,M}}$ and $P_{\text{MD,M}}$, where the estimates are denoted by $\hat{P}_{\text{FA,M}}$ and $\hat{P}_{\text{MD,M}}$, respectively.
\end{enumerate}

The FC needs to make a decision with these unknown parameters which is known as the composite hypothesis testing problem. Since the minimization of \eqref{eq:straightforward_obj_opt} is not tractable, we explore different ways to circumvent this issue. One way is to start by estimating the legitimacy of the robots using trust values only and assuming that the upper bound on the number of malicious robots in the network is known in order to make \eqref{eq:straightforward_obj_opt} tractable. Then, we can ignore the measurements from robots deemed to be malicious and choose the decision regions $\mathcal{D}_0$ and $\mathcal{D}_1$ using the measurements from the remaining robots. This approach leads us to the formulation in \cref{prob:prob1}.

\begin{problem} \label{prob:prob1}
Assume that the FC first estimates the identities of the robots in the network, i.e., it determines $\hat{\boldsymbol{t}}$, solely using the vector of trust values $\boldsymbol{\alpha}$. Then, the FC makes a decision about the hypothesis using only the vector of measurements $\mathbf{Y}$, from robots it identifies as legitimate. Given an upper bound $\bar{m}$ on the proportion of malicious robots in the network, we wish to determine a strategy for the FC that minimizes the following worst-case scenario under these assumptions:
\begin{flalign}
   \min_{\mathcal{D}}\max_{P_{\text{FA,M}},P_{\text{MD,M}},\boldsymbol{t}:\sum_{i\in\mathcal{N}}t_i\leq \bar{m}N} P_{\text{e}}(\mathcal{D},\boldsymbol{t},P_{\text{\emph{FA,M}}},P_{\text{\emph{MD,M}}}).
\end{flalign}

\end{problem}

The definition in \cref{prob:prob1} requires an approach that estimates the trustworthiness of a robot $i$ using only the trust value $a_i$ associated with that robot while assuming a known upper bound on the proportion of malicious robots. However, it is natural to seek additional information about the trustworthiness of the robots that can be obtained from the random measurement vector $\boldsymbol{y}$. Following this intuition, we seek a decision rule that estimates the unknown parameters in the system which are $\boldsymbol{t}$, $P_{\text{FA,M}}$, and $P_{\text{MD,M}}$ as well as the hypothesis $\mathcal{H}_0$ or $\mathcal{H}_1$ jointly, without requiring any known upper bound on the proportion of malicious robots. A common approach to hypothesis testing with unknown parameters is to use the generalized likelihood ratio test \cite{kay_2008}, that is
\begin{flalign}\label{eq:GLRT_ML_decision}
\frac{p(\boldsymbol{z};\hat{\theta}_1,\mathcal{H}_1)}{p(\boldsymbol{z};\hat{\theta}_0,\mathcal{H}_0)}\underset{\mathcal{H}_{0}}{\overset{\mathcal{H}_{1}}{\gtrlesseqslant}} \:\frac{\Pr(\Xi=0)}{\Pr(\Xi=1)} \triangleq {\gamma}_{\text{AG}},
\end{flalign}
where $\hat{\theta}_1$ is the maximum likelihood estimator (MLE) of the unknown parameter $\theta_1$ assuming $\Xi=1$ and $\hat{\theta}_0$ is the MLE of $\theta_0$ assuming $\Xi=0$. For our problem, $\boldsymbol{z}=(\boldsymbol{a},\boldsymbol{y})$, $\theta_1 = (\boldsymbol{t},P_{\text{MD,M}})$, and
$\theta_0 = (\boldsymbol{t},P_{\text{FA,M}}),$ which results in the following formulation of the test 
\begin{flalign}\label{eq:GLRT_ML_decision_joint}
\frac{\max_{\boldsymbol{t}\in\{0,1\}^N,P_{\text{MD,M}}\in[0,1]}\Pr(\boldsymbol{a},\boldsymbol{y}|\mathcal{H}_1,\boldsymbol{t},P_{\text{MD,M}})}{\max_{\boldsymbol{t}\in\{0,1\}^N,P_{\text{FA,M}}\in[0,1]}\Pr(\boldsymbol{a},\boldsymbol{y}|\mathcal{H}_0,\boldsymbol{t},P_{\text{FA,M}})}\underset{\mathcal{H}_{0}}{\overset{\mathcal{H}_{1}}{\begin{smallmatrix}>\\\leqslant\end{smallmatrix}}} \gamma_{\text{AG}}.
\end{flalign}

Note that in this setup the vector $\boldsymbol{t}$ is a parameter, thus, we do not make any prior assumption on its distribution. Calculating the MLE in the numerator and denominator in \eqref{eq:GLRT_ML_decision_joint} is not trivial since the unknown $\boldsymbol{t}$ is a discrete multidimensional variable while $P_{\text{MD,M}}$ and $P_{\text{FA,M}}$ are continuous variables. Doing this in a tractable way leads us to the formulation in \cref{prob:prob2}.

\begin{problem} \label{prob:prob2}
Find a computationally tractable algorithm that calculates the GLRT given in \eqref{eq:GLRT_ML_decision_joint}.
\end{problem}

In the next section we propose solutions to these problems. Then, we investigate the performance of both methods in Section \ref{sec:Results}, and conclude the paper in Section \ref{sec:conclusion}.

\section{Approach}
In this section we present two different approaches: one approach to solve \cref{prob:prob1} and another to solve \cref{prob:prob2}. The first approach, called the Two Stage Approach, finds the optimum decision rule that solves \cref{prob:prob1}. The second approach, called the Adversarial Generalized Likelihood Ratio Test (A-GLRT) uses both the trust values and measurements simultaneously to arrive at a final decision while estimating the unknown parameters using the maximum likelihood estimation rule. The A-GLRT approach addresses \cref{prob:prob2}. The Two Stage Approach is shown to be computationally faster than the A-GLRT, but the A-GLRT attains a lower empirical probability of error. 

\subsection{Two Stage Approach Algorithm}
\label{sec:2SA}

In this section we present an intuitive approach where we separate the detection scheme into two stages where 1) a decision is made about the trustworthiness of each individual robot $i$ based on the received value $\alpha_i$, and then 2) only the measurements $Y_i$ from robots that are trusted are used to choose $\mathcal{H}_0$ or $\mathcal{H}_1$.

\paragraph{Detection of Trustworthy Robots}
We utilize the Likelihood Ratio Test (LRT) to detect \textit{legitimate} robots. This test is guaranteed to have minimal missed detection probability (i.e., detecting a legitimate robot as malicious) for a given false alarm probability (i.e., detecting a malicious robot as legitimate) \cite[Chapter 3]{kay_2008}.

The FC decides which robots to trust using the LRT decision rule
\begin{flalign}\label{eq:classification_decision}
\frac{p_{\alpha}(a_i|t_i=1)}{p_{\alpha}(a_i|t_i=0)}\underset{\hat{t}_i = 0}{\overset{\hat{t}_i = 1}{\gtrless}} \gamma_t,
\end{flalign}
where $\gamma_t$ is a threshold value that we wish to optimize. Note that when $\gamma_t=1$ \eqref{eq:classification_decision} is equivalent to a maximum likelihood detection.

The FC decides who to trust and stores it in the vector $\mathbf{\hat{t}}$, where $\hat{t}_i = 1$ if the FC chooses to trust the robot, and $\hat{t}_i = 0$ otherwise. In the case of equality a random decision is made where the FC chooses $\hat{t}_i = 1$ with probability $p_t$ and the FC chooses $\hat{t}_i = 0$ with probability $1-p_t$, where $p_t$ is another parameter to be optimized. 
This leads to the following trust probabilities, where $P_{\text{trust,L}}(\gamma_t,p_t)$ is the probability of trusting a legitimate robot, and $P_{\text{trust,M}}(\gamma_t,p_t)$ is the probability of trusting a malicious robot:
\begin{equation}
\begin{aligned}
     P_{\text{trust,L}}(\gamma_t,p_t)&= \Pr\left(\frac{p_{\alpha}(a_i|t_i=1)}{p_{\alpha}(a_i|t_i=0)}>\gamma_t|t_i=1\right)\\
     &+p_t  \Pr\left(\frac{p_{\alpha}(a_i|t_i=1)}{p_{\alpha}(a_i|t_i=0)}=\gamma_t|t_i=1\right),\\
    P_{\text{trust,M}}(\gamma_t,p_t)&= \Pr\left(\frac{p_{\alpha}(a_i|t_i=1)}{p_{\alpha}(a_i|t_i=0)}>\gamma_t|t_i=0\right)\\
     &+p_t  \Pr\left(\frac{p_{\alpha}(a_i|t_i=1)}{p_{\alpha}(a_i|t_i=0)}=\gamma_t|t_i=0\right).
     \label{eq:trust_prob_two_stage}
\end{aligned}
\end{equation}

The error probability $P_{\text{e}}$ at the FC is affected by the trustworthiness classification. That is, if a legitimate robot~$i$ is classified as malicious the FC discards its measurement $Y_i$, which increases the error probability since fewer measurements are used in the FC decision making. On the other hand, if a malicious robot is classified as legitimate it can increase the error probability by sending falsified measurements to the FC. For that reason, we look to optimize the trustworthiness classification to balance these two conflicting scenarios. Determining the best $\gamma_t$ and $p_t$ to minimize the overall error probability of the hypothesis detection by the FC is the main focus of this section.

\paragraph{Detecting the Event $\Xi$} 
To determine a hypothesis $\mathcal{H}$ on the event $\Xi$, the FC only considers the measurements it receives from robots that it classifies as legitimate in the first stage, i.e., $i: \hat{t}_i=1$. Equivalently, the FC discards all the received measurements of robots it classifies as malicious. 
Then, the FC uses the following decision rule:
\begin{flalign}\label{eq:detection_p_e_FC_legitimate_assumption}
\frac{\prod_{\{i:\hat{t}_i=1\}}P_{\text{MD},\text{L}}^{1-y_i}(1-P_{\text{MD},\text{L}})^{y_i}}{\prod_{\{i:\hat{t}_i=1\}}(1-P_{\text{FA},\text{L}})^{1-y_i}P_{\text{FA},\text{L}}^{y_i}}\underset{\mathcal{H}_{0}}{\overset{\mathcal{H}_{1}}{\gtrless}}\frac{\Pr(\Xi=0)}{\Pr(\Xi=1)} = \exp(\gamma_{\text{TS}}),
\end{flalign}
where $\exp(\gamma_{\text{TS}})$ is the exponential function with respect to $\gamma_{\text{TS}}$, and it is a constant decision threshold. We set $\frac{\Pr(\Xi=0)}{\Pr(\Xi=1)} = \exp(\gamma_{\text{TS}})$ so that when we take the logarithm in later expressions we can express the resultant decision threshold as $\gamma_{\text{TS}}$ for ease of exposition. This decision rule is commonly used in standard binary hypothesis testing problems where no malicious robots are present, and will be referred to as the \emph{standard binary hypothesis decision rule}. The standard binary hypothesis decision rule is optimal in a system with no malicious robots, i.e., ${\cal M}=\emptyset$, and thus we attempt to approximate the standard binary hypothesis decision rule by first removing information from all robots deemed to be malicious. However, since there may be detection errors in the first stage which classifies legitimate and malicious robots, the threshold $\gamma_t$ and tie-break probability $p_t$ should  balance the need to exclude malicious robots from participating in the test \eqref{eq:detection_p_e_FC_legitimate_assumption} with the need to allow legitimate robots to participate in the test \eqref{eq:detection_p_e_FC_legitimate_assumption} and contribute their truthful measurements to decrease the probability of error resulting from \eqref{eq:detection_p_e_FC_legitimate_assumption}. In what follows we show how to optimize the threshold $\gamma_t$ and tie-break probability $p_t$ by first computing the probability of error of the FC using the Two Stage Approach.

Recalling the Neyman-Pearson Lemma \cite{kay_2008}, we have that \eqref{eq:classification_decision}
minimizes the missed detection probability for a desired false alarm probability of misclassifying robots. This false alarm probability dictates the value of the threshold $\gamma_t$. After the FC discards robot measurements that it does not trust, the decision rule \eqref{eq:detection_p_e_FC_legitimate_assumption} leads to the following false alarm and missed detection error probabilities,
 \begin{equation}
 \begin{aligned}
    &P_{\text{FA}}(\gamma_t,p_t,\mathbf{t},P_{\text{FA,M}}) \\
    & =\Pr\Big( \sum_{i=1}^N \hat{t}_i [w_{1,\text{L}} y_i - w_{0,\text{L}}(1-y_i)] \geq \gamma_{\text{TS}} \\
     &\hspace{4.5cm}| \mathcal{H}_0, \gamma_t,p_t,\mathbf{t},P_{\text{FA,M}} \Big), \\ 
    &P_{\text{MD}}(\gamma_t,p_t,\mathbf{t},P_{\text{MD,M}}) \\
    &= \Pr\Big( \sum_{i=1}^N \hat{t}_i [w_{1,\text{L}} y_i - w_{0,\text{L}}(1-y_i)] < \gamma_{\text{TS}} \\
    &\hspace{4.5cm}| \mathcal{H}_1 ,\gamma_t,p_t,\mathbf{t},P_{\text{MD,M}}\Big),
        \label{eq:fa_md_total}
    \end{aligned}
    \end{equation}
where
\begin{equation}
    w_{1,\text{L}} = \log\left( \frac{1 - P_{\text{MD},\text{L}}}{P_{\text{FA},\text{L}}} \right), \quad w_{0,\text{L}} = \log\left( \frac{1 - P_{\text{FA},\text{L}}}{P_{\text{MD},\text{L}}} \right).
    \label{eq:w1_w0}
\end{equation}
Consequently, the overall error probability at the FC is:
\begin{equation}
\begin{aligned}
    &P_{\text{e}}(\gamma_t,p_t,\mathbf{t},P_{\text{FA,M}},P_{\text{MD,M}})  \\ 
    &= \Pr(\Xi=0) P_{\text{FA}}(\gamma_t,p_t,\mathbf{t},P_{\text{FA,M}}) \\ 
    &+ \Pr(\Xi=1) P_{\text{MD}}(\gamma_t,p_t,\mathbf{t},P_{\text{MD,M})}.
    \label{eq:error_prob_2stage}
\end{aligned}
\end{equation}
   
We seek to minimize the probability of error \eqref{eq:error_prob_2stage} for the decision rule \eqref{eq:detection_p_e_FC_legitimate_assumption} by minimizing the false alarm and missed detection probabilities. Any sequence of $0$'s and $1$'s can occur for the  detected trust vector $\mathbf{\hat{t}}$, each yielding a different error probability, so the error probability must be calculated for each possible vector $\mathbf{\hat{t}}$, along with each possible vector $\mathbf{y}$. Unfortunately, this computation scales exponentially with the number of robots, $N$. Furthermore, the true trust vector $\mathbf{t}$ and the probabilities of false alarm and missed detection of the malicious robots are unknown, i.e., $P_{\text{FA,M}}$ and $P_{\text{MD,M}}$, therefore, they cannot be used in minimizing  \eqref{eq:error_prob_2stage}.

To this end, we derive analytical guarantees regarding the error probability of the overall detection performance of the two-stage approach as follows. We minimize the worst-case probability of error of the FC over all the possible trust vectors $\mathbf{t}\in\{0,1\}^N$ and false alarm and missed detection probabilities $P_{\text{FA,M}}$ and $P_{\text{MD,M}}$, respectively, in the interval $[0,1]$. Then, we minimize this worst-case error probability by choosing the best threshold $\gamma_t$, i.e., choose $\gamma_t = \gamma_t^*$ and tie-break probability $p_t = p_t^*$ where
\begin{equation}
    (\gamma_t^*,p_t^*) = \argmin_{\gamma_t,p_t} \max_{\mathbf{t},P_{\text{FA,M}},P_{\text{MD,M}}} P_{\text{e}}(\gamma_t,p_t,\mathbf{t},P_{\text{FA,M}},P_{\text{MD,M}}).
    \label{eq:opt_wors_case_two_stage}
\end{equation}

To this end, we must first determine the $P_{\text{FA,M}}, P_{\text{MD,M}}, \mathbf{t}$ that maximize $P_{\text{e}}$. In the remainder of this section, we assume that the proportion of malicious robots to expect in the network, denoted by $m$, is known, or we choose an upper bound for it $(\bar{m})$.

\begin{lemma} \label{lem:P_FA_M}
If $P_{\text{\emph{FA,L}}} < 0.5$ and $P_{\text{\emph{MD,L}}} < 0.5$, then the probability of false alarm and missed detection of the FC \eqref{eq:fa_md_total} is maximized for the two stage approach when malicious robots choose $P_{\text{FA,M}} = P_{\text{MD,M}} = 1$, for any vector $\boldsymbol{t}\in\{0,1\}^N$.
\end{lemma}
The proof of Lemma \ref{lem:P_FA_M} can be found in Appendix \ref{sec:proof_P_FA_M}. 

\begin{lemma} \label{lem:t_max}
Let $\bar{\mathbf{t}}$ be the worst-case vector $\mathbf{t}$, i.e., the vector $\mathbf{t}$ that maximizes the probability of error \eqref{eq:error_prob_2stage}. If $P_{\text{\emph{FA,L}}} < 0.5$, $P_{\text{\emph{MD,L}}} < 0.5$, and $P_{\text{\emph{FA,M}}} = P_{\text{\emph{MD,M}}} = 1$, then the probability of error $P_{\text{e}}(\gamma_t,p_t,\bar{\mathbf{t}},1,1)$ is maximized when $\bar{\mathbf{t}}$ contains the maximum number of malicious robots, i.e., $\sum_{i\in\mathcal{N}}\bar{t}_i=\bar{m}N$.
\end{lemma}
\begin{proof}
By Lemma \ref{lem:P_FA_M} the probability of false alarm and missed detection \eqref{eq:fa_md_total} are maximized when a robot is trusted and its measurement reports the wrong hypothesis ($Y_i = 1 | \mathcal{H}_0$ or $Y_i = 0 | \mathcal{H}_1$). Since the optimal policy for malicious robots is to report the wrong hypothesis with probability $1$ (Lemma \ref{lem:P_FA_M}), any robot increases the false alarm and missed detection probability of the FC when it is malicious instead of legitimate. Thus, the probability of error $P_{\text{e}}(\gamma_t,p_t,\mathbf{t},1,1)$ is maximized when the proportion of malicious robots, $m$, is maximized, i.e., when $\bar{\mathbf{t}}$ has $\bar{m}N$ malicious robots, where $\bar{m}$ is the upper bound on the proportion of malicious robots in the network.
\end{proof}

Utilizing Lemma \ref{lem:t_max}, we calculate the exact probability of error for the FC for the worst-case attack where $\mathbf{t} = \bar{\mathbf{t}}$ and $P_{\text{FA,M}} = P_{\text{MD,M}} = 1$. In order to compute the probability of error exactly, we must compute the probability of false alarm and missed detection \eqref{eq:fa_md_total}.
Let $k_{\text{L}} \in K_{\text{L}}$ be the number of legitimate robots trusted by the FC, where $K_{\text{L}} = \{ 0, \dots, (1-\bar{m})N \}$. Similarly, let $k_{\text{M}} \in K_{\text{M}}$ be the number of malicious robots trusted by the FC, where $K_{\text{M}} = \{ 0, \dots, \bar{m}N \}$. Let $S_N$ represent the left side of the inequalities in \eqref{eq:fa_md_total} given by:
\[S_\text{N} = \sum_{i=1}^N \hat{t}_i [w_{1,\text{L}} y_i - w_{0,\text{L}}(1-y_i)].\] 
Using the law of total probability, the false alarm  probability at the FC is given by
\begin{equation}
\begin{aligned}
    \resizebox{0.26\hsize}{!}{$P_{\text{FA}}(\gamma_t,p_t,\bar{\mathbf{t}},1)$} &= \resizebox{0.68\hsize}{!}{$\sum_{k_{\text{L}} \in K_{\text{L}}, k_{\text{M}} \in K_{\text{M}}} \Pr(K_{\text{L}} = k_{\text{L}}) \Pr(K_{\text{M}} = k_{\text{M}})$} \\ & \qquad \qquad \quad \quad \quad \cdot \resizebox{0.42\hsize}{!}{$P_{\text{FA}}(S_\text{N} \geq \gamma_{\text{TS}} | \mathcal{H}_0, k_{\text{L}},k_{\text{M}})$}.
\end{aligned}
\end{equation}
Similarly, the probability of missed detection of the FC is given by 
\begin{equation}
\begin{aligned}
    \resizebox{0.26\hsize}{!}{$P_{\text{MD}}(\gamma_t,p_t,\bar{\mathbf{t}},1)$} &= \resizebox{0.68\hsize}{!}{$\sum_{k_{\text{L}} \in K_{\text{L}}, k_{\text{M}} \in K_{\text{M}}} \Pr(K_{\text{L}} = k_{\text{L}}) \Pr(K_{\text{M}} = k_{\text{M}})$} \\ & \qquad \qquad \quad \quad \quad \cdot \resizebox{0.42\hsize}{!}{$P_{\text{MD}}(S_\text{N} < \gamma_{\text{TS}} | \mathcal{H}_1, k_{\text{L}},k_{\text{M}})$}.
\end{aligned}
\end{equation}
The probability of false alarm for a particular instantiation of $k_{\text{L}}$ and $k_{\text{M}}$ can be written as a function of the Binomial Cumulative Distribution Function:
\begin{equation}
    \begin{aligned}
        &P_{\text{FA}}(S_\text{N} \geq \gamma_{\text{TS}} | \mathcal{H}_0, k_{\text{L}},k_{\text{M}}) \\ &= \resizebox{0.87\hsize}{!}{$\Pr\left( \sum_{i: \{\hat{t}_i=1,t_i=1\}} y_i \geq \frac{ \gamma_{\text{TS}} - k_{\text{M}}w_{1,\text{L}} + k_{\text{L}}w_{0,\text{L}} }{w_{0,\text{L}} + w_{1,\text{L}}} | \mathcal{H}_0,k_{\text{L}},k_{\text{M}}, \right)$}, \\ &= 1 - F_{\text{b}}\left( \lceil \frac{ \gamma_{\text{TS}} - k_{\text{M}}w_{1,\text{L}} + k_{\text{L}}w_{0,\text{L}} }{w_{0,\text{L}} + w_{1,\text{L}}} \rceil ; P_{\text{FA,L}},k_{\text{L}}\right),
    \end{aligned}
    \label{eq:P_FA_kl_km}
\end{equation}
where $F_{\text{b}}(x;p,n) = \sum_{i=0}^x \binom{n}{i} p^i (1-p)^{n-i}$ is the Binomial Cumulative Distribution Function evaluated at $x$ for $n$ variables and success probability $p$. Similarly, for the probability of missed detection we have that
\begin{equation}
\begin{aligned}
    &P_{\text{MD}}(S_\text{N} < \gamma_{\text{TS}} | \mathcal{H}_1, k_{\text{L}},k_{\text{M}})= \\ &   F_{\text{b}}\left(\lceil \frac{ \gamma_{\text{TS}} + k_{\text{M}}w_{1,\text{L}} + k_{\text{L}}w_{0,\text{L}} }{w_{0,\text{L}} + w_{1,\text{L}}} \rceil - 1 ; 1-P_{\text{MD,L}},k_{\text{L}}\right).
\end{aligned}
\end{equation}

Recall \eqref{eq:trust_prob_two_stage}. We note that these probabilities depend on the distribution of the robot's vector of trust values $\boldsymbol{a}$. Then, we have that
\begin{equation}
    \begin{aligned}
        \Pr(K_{\text{L}} = k_{\text{L}}) &= \Pr\left( \sum_{i \in \mathcal{L}} \hat{t}_i = k_{\text{L}} \right) \\ &= f_{\text{b}}(k_{\text{L}} ; P_{\text{trust,L}}(\gamma_t,p_t), (1-\bar{m})N), \\ \Pr(K_{\text{M}} = k_{\text{M}}) &= \Pr\left( \sum_{i \in \bar{\mathcal{M}}} \hat{t}_i = k_{\text{M}} \right) \\ &= f_{\text{b}}(k_{\text{M}} ; P_{\text{trust,M}}(\gamma_t,p_t), \bar{m}N),
    \end{aligned}
\end{equation}
where $f_{\text{b}}(x ; p,n) = \binom{n}{x} p^x (1-p)^{n-x}$ is the Binomial probability distribution function evaluated at $x$ for $n$ variables and success probability $p$. Thus, the probability of false alarm and missed detection are
\begin{equation}
    \begin{aligned}
        &P_{\text{FA}}(\gamma_t,p_t,\bar{\mathbf{t}},1) \\
        &= \sum_{k_{\text{L}} \in K_{\text{L}}, k_{\text{M}} \in K_{\text{M}}}  f_{\text{b}}(k_{\text{L}} ; P_{\text{trust,L}}(\gamma_t,p_t), (1-\bar{m})N) \cdot \\ 
        &\hspace{2cm} f_{\text{b}}(k_{\text{M}} ; P_{\text{trust,M}}(\gamma_t,p_t), \bar{m}N) \cdot \\ 
        &\hspace{2cm} P_{\text{FA}}(S_\text{N} \geq \gamma_{\text{TS}} | \mathcal{H}_0, k_{\text{L}},k_{\text{M}}), \\ 
        &P_{\text{MD}}(\gamma_t,p_t,\bar{\mathbf{t}},1) \\
        &= \sum_{k_{\text{L}} \in K_{\text{L}}, k_{\text{M}} \in K_{\text{M}}}  f_{\text{b}}(k_{\text{L}} ; P_{\text{trust,L}}(\gamma_t,p_t), (1-\bar{m})N) \cdot \\ 
        &\hspace{2cm} f_{\text{b}}(k_{\text{M}} ; P_{\text{trust,M}}(\gamma_t,p_t), \bar{m}N) \cdot \\ 
        &\hspace{2cm} P_{\text{MD}}(S_\text{N} < \gamma_{\text{TS}} | \mathcal{H}_1, k_{\text{L}},k_{\text{M}}).
    \end{aligned}
    \label{eq:P_FA_MD_WC}
\end{equation}

Therefore, we have the total error probability
\begin{equation}
\begin{aligned}
    P_{\text{e}}(\gamma_t,p_t,\bar{\mathbf{t}},1,1) = &\Pr(\Xi=0) P_{\text{FA}}(\gamma_t,p_t,\bar{\mathbf{t}},1) + \\ & \Pr(\Xi=1) P_{\text{MD}}(\gamma_t,p_t,\bar{\mathbf{t}},1),
\end{aligned}
\label{eq:P_e_2SA_alg}
\end{equation}
and we can choose the thresholds $\gamma_t$ and $p_t$ that minimize the expression. Once we have chosen the thresholds $\gamma_t$ and $p_t$, the rest of the two stage approach becomes a standard binary hypothesis testing problem.

\begin{lemma}\label{lemma:sufficient_discrete_gamma}
Denote \[\Gamma_t: = \left\{ \frac{p_{\alpha}(a|t_i=1)}{p_{\alpha}(a|t_i=0)} \right\}_{a\in\mathcal{A}}.\] 
Then, the minimal value of \eqref{eq:opt_wors_case_two_stage} with respect to $\gamma_t$ can be achieved by $\gamma_t\in\Gamma_t$.
\end{lemma}
\begin{proof}
The proof follows directly from the finiteness of the  set $\mathcal{A}$ and since $p_t$ can take values in the interval $[0,1]$.
\end{proof}

\begin{algorithm}[h!]
\caption{Two Stage Approach \\ Input: $P_{\text{FA},\text{L}}$, $P_{\text{MD},\text{L}}$, $\hat{P}_{\text{FA},\text{M}} = \hat{P}_{\text{MD},\text{M}} = 1$, $\Pr(\Xi=0)$, $\Pr(\Xi=1)$, $\mathbf{y}$, $\boldsymbol{a}$, $\bar{\mathbf{t}}$, $\Gamma_t$, $\delta_p$ \\ Output: Decision $\mathcal{H}_0$ or $\mathcal{H}_1$}
\label{alg:2stage_discrete}
\begin{algorithmic}[1]
\State Set $\Gamma_p = \{0,\delta_p,2\delta_p,\dots,1\}$.
\State Set $\gamma_{t,\text{temp}}=0$, $p_{t,\text{temp}}=0$, $P_{\text{e,temp}}=2$.
\ForAll{$\hat{\gamma}_t \in \Gamma_t$, $\hat{p}_t \in \Gamma_p$}
\State Compute $P_{\text{trust,L}}(\hat{\gamma}_t,\hat{p}_t)$, $P_{\text{trust,M}}(\hat{\gamma}_t,\hat{p}_t)$ by \eqref{eq:trust_prob_two_stage}.
\State Compute $P_{\text{FA}}(\hat{\gamma}_t,\hat{p}_t,\bar{\mathbf{t}},1)$, $P_{\text{MD}}(\hat{\gamma}_t,\hat{p}_t,\bar{\mathbf{t}},1)$ by \eqref{eq:P_FA_MD_WC}.
\State Compute $P_{\text{e}}(\hat{\gamma}_t,\hat{p}_t,\bar{\mathbf{t}},1,1)$ by \eqref{eq:P_e_2SA_alg}. 
\If{$P_{\text{e}}(\hat{\gamma}_t,\hat{p}_t,\bar{\mathbf{t}},1,1)<P_{\text{e,temp}}$}
\State Set $(\gamma_{t,\text{temp}}, p_{t,\text{temp}})=(\hat{\gamma}_t,\hat{p}_t)$.
\State Set $P_{\text{e,temp}}=P_{\text{e}}(\hat{\gamma}_t,\hat{p}_t,\bar{\mathbf{t}},1,1)$.
\EndIf
\EndFor
\State Set $(\gamma_t, p_t) = (\gamma_{t,\text{temp}}, p_{t,\text{temp}})$.
\State Determine the vector $\mathbf{\hat{t}}$ using \eqref{eq:classification_decision}.
\State Determine decision using \eqref{eq:detection_p_e_FC_legitimate_assumption}.
\State Return decision $\mathcal{H}_0$ or $\mathcal{H}_1$.
\end{algorithmic}
\end{algorithm}

Algorithm \ref{alg:2stage_discrete} explains the two stage approach step-by-step. Algorithm \ref{alg:2stage_discrete} takes a set $\Gamma_t$ as input. Then, for each $\hat{\gamma}_t \in \Gamma_t$ and each $\hat{p}_t \in \Gamma_p$ we compute $P_{\text{trust,L}}(\hat{\gamma}_t,\hat{p}_t)$, $P_{\text{trust,M}}(\hat{\gamma}_t,\hat{p}_t)$, as well as $P_{\text{FA}}(\hat{\gamma}_t,\hat{p}_t,\bar{\mathbf{t}},1)$ and $P_{\text{MD}}(\hat{\gamma}_t,\hat{p}_t,\bar{\mathbf{t}},1)$. Then we compute the probability of error at the FC for the given $\hat{\gamma}_t$ and $\hat{p}_t$. The $\hat{\gamma}_t$ and $\hat{p}_t$ that yields the minimum probability of error is then used in the decision rule in \eqref{eq:classification_decision} to determine which robots to trust or not trust (vector $\mathbf{\hat{t}}$). Finally, we use the chosen vector $\mathbf{\hat{t}}$ to make a decision using the standard binary hypothesis decision rule \eqref{eq:detection_p_e_FC_legitimate_assumption}.

Determining the threshold value $\gamma_t$ and tie-break probability $p_t$ requires computing the probability of error $|\Gamma_t|\cdot|\Gamma_p|$ times, where $|\cdot|$ represents the cardinality of the set. However, this only needs to be computed once, and then the returned $\gamma_t$ and $p_t$ can be used to run each subsequent hypothesis test. With a given $\gamma_t$ and $p_t$, the hypothesis test requires $\mathcal{O}(N)$ comparisons.

\begin{theorem}\label{thm:2SA_optimum_decision_rule}
Assume that the FC uses the decision rule in \eqref{eq:classification_decision} to detect malicious robots, and then uses the decision rule \eqref{eq:detection_p_e_FC_legitimate_assumption}. Then Algorithm \ref{alg:2stage_discrete} chooses the threshold value $\gamma_t$ and tie-break probability $p_t$ that minimize the worst-case probability of error of the FC up to a  discretization distance \[d(\delta_p):=\min_{p_t\in\Gamma_p}P_{\text{e}}(\gamma_t^*,p_t,\bar{\mathbf{t}},1,1)-P_{\text{e}}(\gamma_t^*,p_t^*,\bar{\mathbf{t}},1,1).\]
Furthermore, $d(\delta_p)\rightarrow 0$ as $\delta_p\rightarrow 0$.
\end{theorem}
\begin{proof}
The goal is to minimize the worst-case probability of error of the FC, i.e.,
\begin{equation}
    \min_{\gamma_t,p_t} \max_{\mathbf{t},P_{\text{FA,M}},P_{\text{MD,M}}} P_{\text{e}}(\gamma_t,p_t,\mathbf{t},P_{\text{FA,M}},P_{\text{MD,M}}).
\end{equation}
Let $\bar{P}_e$ be the worst-case probability of error computed using the worst-case probability of false alarm and missed detection from \eqref{eq:P_FA_MD_WC}. Furthermore, let $\bar{\mathbf{t}}$ be the worst-case vector $\mathbf{t}$. Using the results from Lemmas \ref{lem:P_FA_M}, \ref{lem:t_max} and \eqref{eq:P_FA_MD_WC} we upper bound the error probability using the worst-case error probability:
\begin{equation}
    \begin{aligned}
        \min_{\gamma_t,p_t} \max_{\mathbf{t},P_{\text{FA,M}},P_{\text{MD,M}}} P_{\text{e}}&(\gamma_t,p_t,\mathbf{t},P_{\text{FA,M}},P_{\text{MD,M}}) \\ &= \min_{\gamma_t,p_t} \max_{\mathbf{t}} P_{\text{e}}(\gamma_t,p_t,\mathbf{t},1,1), \\ 
        &= \min_{\gamma_t,p_t} \bar{P}_e(\gamma_t,p_t,\bar{\mathbf{t}},1,1).
    \end{aligned}
\end{equation}
The equality in the first line directly follows from Lemma \ref{lem:P_FA_M}. The second line follows from the first by inserting the worst-case vector $\mathbf{t}$ as the one that maximizes the probability of error $P_{\text{e}}$ (Lemma \ref{lem:t_max}).

Additionally, by Lemma \ref{lemma:sufficient_discrete_gamma}, it is  sufficient to optimize $\gamma_t$ over the set  $\Gamma_t$. 
Now, since we optimize $p_t$ using a line search, we may not necessarily find an optimal pair $(\gamma_t^*,p_t^*)$. However, we can upper bound the distance from the optimal solution for the worst case scenario by:
\begin{flalign}
    &\min_{\gamma_t\in\Gamma_t,p_t\in\Gamma_p}P_{\text{e}}(\gamma_t,p_t,\bar{\mathbf{t}},1,1)-P_{\text{e}}(\gamma_t^*,p_t^*,\bar{\mathbf{t}},1,1)\nonumber\\
    &\leq \min_{p_t\in\Gamma_p}P_{\text{e}}(\gamma_t^*,p_t,\bar{\mathbf{t}},1,1)-P_{\text{e}}(\gamma_t^*,p_t^*,\bar{\mathbf{t}},1,1)\nonumber\\
    &=d(\delta_p).
\end{flalign}
For every fixed $\gamma_t$, the function $P_{\text{e}}(\gamma_t,p_t,\bar{\mathbf{t}},1,1)$ is a polynomial function of $p_t$, therefore, it is continuous in $p_t$ (over the interval $p_t\in[0,1]$). 
Consequently, 
$d(\delta_p)\rightarrow 0$ as $\delta_p\rightarrow 0$. 
\end{proof}

\subsection{A-GLRT Algorithm} \label{sec:GLRT}
The main purpose of this section is to construct an efficient algorithm that implements the GLRT in \eqref{eq:GLRT_ML_decision_joint}. We can simplify \eqref{eq:GLRT_ML_decision_joint} by recalling that given the true trustworthiness of a robot $t_i$ and the true hypothesis $\mathcal{H}$, the trust value $\alpha_i$ and the measurement $Y_i$ are statistically independent. 
Thus,
\begin{flalign}
&\Pr(\boldsymbol{a},\boldsymbol{y}|\mathcal{H}_1,\boldsymbol{t},P_{\text{MD,M}})\nonumber\\
&=\Pr(\boldsymbol{a}|\mathcal{H}_1,\boldsymbol{t},P_{\text{MD,M}})\Pr(\boldsymbol{y}|\mathcal{H}_1,\boldsymbol{t},P_{\text{MD,M}}),\\
&\Pr(\boldsymbol{a},\boldsymbol{y}|\mathcal{H}_0,\boldsymbol{t},P_{\text{FA,M}})\nonumber\\
&=\Pr(\boldsymbol{a}|\mathcal{H}_0,\boldsymbol{t},P_{\text{FA,M}})\Pr(\boldsymbol{y}|\mathcal{H}_0,\boldsymbol{t},P_{\text{FA,M}}).
\end{flalign}
Furthermore, the trust value $\alpha_i$ is independent of the  true hypothesis $\mathcal{H}$. Thus,
\begin{flalign}
   \Pr(\boldsymbol{a}|\mathcal{H}_1,\boldsymbol{t},P_{\text{MD,M}})=\Pr(\boldsymbol{a}|\mathcal{H}_0,\boldsymbol{t},P_{\text{FA,M}})=\Pr(\boldsymbol{a}|\boldsymbol{t}). 
\end{flalign}

Hence, we obtain
\begin{flalign}\label{eq:GLRT_ML_decision_independent}
\frac{\max_{\boldsymbol{t}\in\{0,1\}^N,P_{\text{MD,M}}\in[0,1]}\Pr(\boldsymbol{a}|\boldsymbol{t})\Pr(\boldsymbol{y}|\mathcal{H}_1,\boldsymbol{t},P_{\text{MD,M}})}{\max_{\boldsymbol{t}\in\{0,1\}^N,P_{\text{FA,M}}\in[0,1]}\Pr(\boldsymbol{a}|\boldsymbol{t})\Pr(\boldsymbol{y}|\mathcal{H}_0,\boldsymbol{t},P_{\text{FA,M}})}\underset{\mathcal{H}_{0}}{\overset{\mathcal{H}_{1}}{\begin{smallmatrix}>\\\leqslant\end{smallmatrix}}} \gamma_{\text{AG}}.
\end{flalign}
We choose $\gamma_{\text{AG}}=\frac{\Pr(\Xi=0)}{\Pr(\Xi=1)}$ since we do not assume anything about the the prior distribution of $\boldsymbol{t}$. 
The challenging part of using the GLRT in this problem is calculating the maximum likelihood estimations for both numerator and denominator. The unknown $\boldsymbol{t}$ is a discrete multidimensional variable while $P_{\text{MD,M}}$ and $P_{\text{FA,M}}$ are continuous variables restricted to the domain $[0,1]$. Therefore, calculating the MLE is not trivial. The main purpose of this section is to construct an efficient algorithm that implements the GLRT. Due to the symmetry in calculation of the numerator and denominator in \eqref{eq:GLRT_ML_decision_independent}, we focus our discussion on the calculation of the numerator. 

Using Assumption \ref{assumption:trust_val} about the trust values, we obtain the following formulation of $\Pr(\boldsymbol{a}|\boldsymbol{t})$:
$$\Pr(\boldsymbol{a}|\boldsymbol{t}) = \prod_{i=1}^Np_{\alpha}(a_i|t_i).$$ Additionally, we obtain the following equations using the i.i.d assumption about measurements: 
\begin{flalign}
   & \Pr(\boldsymbol{y}|\mathcal{H}_0,\boldsymbol{t},P_{\text{FA,M}})=\prod_{i:t_i=1} P_{\text{FA,L}}^{y_i}\cdot(1-P_{\text{FA,L}})^{1-y_i} \nonumber\\
   &\hspace{3cm}\cdot \prod_{i:t_i=0} P_{\text{FA,M}}^{y_i}\cdot(1-P_{\text{FA,M}})^{1-y_i}\label{eq:assumption_y_given_h0_t},\\
   & \Pr(\boldsymbol{y}|\mathcal{H}_1,\boldsymbol{t},P_{\text{MD,M}})=\prod_{i:t_i=1} (1-P_{\text{MD,L}})^{y_i}\cdot P_{\text{MD,L}}^{1-y_i}\nonumber\\
   &\hspace{3cm}\cdot \prod_{i:t_i=0} (1-P_{\text{MD,M}})^{y_i}\cdot P_{\text{MD,M}}^{1-y_i}. \label{eq:assumption_y_given_h1_t}
\end{flalign}
Using these equations, we write the numerator as:
\begin{equation}
    \begin{aligned}
    \label{eq:GLRT_numerator_maximization}
    \max_{\boldsymbol{t}\in\{0,1\}^N,P_{\text{MD,M}}\in[0,1]}\left\{\prod_{i:t_i=1}p_{\alpha}(a_i|t_i)P_{\text{MD,L}}^{1-y_i}(1-P_{\text{MD,L}})^{y_i} \right. \cdot \\ 
    \left. \prod_{i:t_i=0}p_{\alpha}(a_i|t_i)P_{\text{MD,M}}^{1-y_i}(1-P_{\text{MD,M}})^{y_i} \right\}.
    \end{aligned}
\end{equation}

Since the optimization problem over variables $\boldsymbol{t}$ and $P_{\text{MD,M}}$ at the same time is difficult we can reformulate the problem as two nested optimizations using the Principle of Iterated Suprema \cite[p. 515]{olmsted1959real}, that is: 
$$
\begin{aligned}
    \sup\{f(z,w): z \in \mathcal{Z}, w \in \mathcal{W}\} =
    \sup_{z \in \mathcal{Z}}\{\sup_{w \in \mathcal{W}}\{f(z,w)\}\} \\
    = \sup_{w \in \mathcal{W}}\{\sup_{z \in \mathcal{Z}}\{f(z,w)\}\},
    \end{aligned}
$$
where $f \colon \mathcal{Z}\times \mathcal{W} \to \mathbb{R}$, and $\mathcal{Z},\mathcal{W}\subseteq \mathbb{R}^d$. By the Principle of Iterated Suprema we can calculate the maximization in \eqref{eq:GLRT_ML_decision} in two different ways. 
We rewrite the maximization problem in \eqref{eq:GLRT_numerator_maximization} as:
\begin{equation}
    \begin{aligned}\label{eq:GLRT_numerator_t_first}
    \max_{\boldsymbol{t}\in\{0,1\}^N}\left\{\max_{P_{\text{MD,M}}\in[0,1]}\left\{\prod_{i:t_i=1}p_{\alpha}(a_i|t_i)P_{\text{MD,L}}^{1-y_i}(1-P_{\text{MD,L}})^{y_i} \right. \right. \cdot \\
    \left. \left.\prod_{i:t_i=0}p_{\alpha}(a_i|t_i)P_{\text{MD,M}}^{1-y_i}(1-P_{\text{MD,M}})^{y_i}\right\}\right\}.
    \end{aligned}
\end{equation}

With this formulation, one possible way to calculate the maximization is iterating over all vectors  $\boldsymbol{t}$ in the set $\{0,1\}^N$; then for each $\boldsymbol{t}$, calculating the inner maximization. We show how to calculate this maximization in the following lemma.

\begin{lemma}
\label{lem:lemma_t_given}
Let $\boldsymbol{t}$ and $\boldsymbol{y}$ be given vectors in $\{0,1\}^N$. Assume that $p_{\alpha}(a_i|t_i)$ is known both $t_i=0$ and $t_i=1$, and that $\sum_{i:t_i=0}1>0$. Then,
\begin{equation}
    \begin{aligned}
    &\prod_{i:t_i=1}p_{\alpha}(a_i|t_i)P_{\text{\emph{MD,L}}}^{1-y_i}(1-P_{\text{\emph{MD,L}}})^{y_i} \cdot\\
   & \qquad\prod_{i:t_i=0}p_{\alpha}(a_i|t_i)P_{\text{\emph{MD,M}}}^{1-y_i}(1-P_{\text{\emph{MD,M}}})^{y_i}
    \end{aligned}
\label{eq:numerator_inner_max}
\end{equation}

is maximized by $\widehat{P}_{\text{\emph{MD,M}}}=\frac{\sum_{i:t_i=0}(1-y_i)}{\sum_{i:t_i=0}1}$.
Additionally, if $\sum_{i:t_i=0}1=0$, i.e., $|\{i:t_i=0\}|=0$, any choice $\widehat{P}_{\text{\emph{MD,M}}}\in[0,1]$ maximizes \eqref{eq:numerator_inner_max}.
\end{lemma}
\begin{proof}
First, observe that given the vector $\boldsymbol{t}$,  \eqref{eq:numerator_inner_max} is maximized by MLE of $\prod_{i:t_i=0}p_{\alpha}(a_i|t_i)P_{\text{MD,M}}^{1-y_i}(1-P_{\text{MD,M}})^{y_i}$. 
Furthermore, since
\begin{flalign}
&\prod_{i:t_i=0}p_{\alpha}(a_i|t_i)P_{\text{MD,M}}^{1-y_i}(1-P_{\text{MD,M}})^{y_i}\nonumber\\
&=\left(\prod_{i:t_i=0}p_{\alpha}(a_i|t_i)\right)\left(\prod_{i:t_i=0}P_{\text{MD,M}}^{1-y_i}(1-P_{\text{MD,M}})^{y_i}\right),   
\end{flalign}
it follows that  \eqref{eq:numerator_inner_max} is maximized by the MLE of $\prod_{i:t_i=0}P_{\text{MD,M}}^{1-y_i}(1-P_{\text{MD,M}})^{y_i}$.
\\
This is a well-known estimation problem \cite[Problem 7.8]{kay1993estimation}, that together with the invariance property of the MLE \cite[Theorem 7.2]{kay1993estimation} leads to the optimal estimator 
\[\widehat{P}_{\text{MD,M}}=\frac{\sum_{i:t_i=0}(1-y_i)}{\sum_{i:t_i=0}1}.\]
Note, that this estimator is equal to the empirical missed detection probability of the measurements sent by the malicious robots.
Finally, it is easy to validate that if  $|\{i:t_i=0\}|=0$, any choice $\widehat{P}_{\text{MD,M}}\in[0,1]$ maximizes \eqref{eq:numerator_inner_max}.
\end{proof}

Unfortunately, since the set $\{0,1\}^N$ exponentially with the number of robots in the network, this approach is computationally intractable for large robot networks. Therefore, we look for an alternative solution. Another equivalent formulation of the maximization problem that is obtained by the Principle of Iterated Supremum is
\begin{equation}
    \begin{aligned}\label{eq:GLRT_numerator_pma_first}
    \max_{P_{\text{MD,M}}\in[0,1]}\left\{\max_{\boldsymbol{t}\in\{0,1\}^N}\left\{\prod_{i:t_i=1}p_{\alpha}(a_i|t_i)P_{\text{MD,L}}^{1-y_i}(1-P_{\text{MD,L}})^{y_i} \right. \right. \cdot \\
    \left. \left.\prod_{i:t_i=0}p_{\alpha}(a_i|t_i)P_{\text{MD,M}}^{1-y_i}(1-P_{\text{MD,M}})^{y_i}\right\}\right\},
    \end{aligned}
\end{equation}
where the order of variables that the maximization is taken over is flipped. Since the variable $P_{\text{MD,M}}$ belongs to an uncountably infinite set, it is impossible to perform the maximization with this formulation. However, assuming that we have a given $P_{\text{MD,M}}$, the inner maximization can still be calculated. The following lemma shows how to calculate the inner maximization.

\begin{lemma}
\label{lem:lemma_P_given}
Let $P_{\text{\emph{MD,M}}}$, $\boldsymbol{a}$, and $\mathbf{y}$ be given. Additionally, assume that $p_{\alpha}(a_i|t_i)$ is known for both $t_i=0$ and $t_i=1$.
Let 
$$
\begin{aligned}
    c_{\text{\emph{L}},i}=p_{\alpha}(a_i|t_i)P_{\text{\emph{MD,L}}}^{1-y_i}(1-P_{\text{\emph{MD,L}}})^{y_i}
\end{aligned}
$$ and 
$$
\begin{aligned}
    c_{\text{\emph{M}},i}= p_{\alpha}(a_i|t_i)P_{\text{\emph{MD,M}}}^{1-y_i}(1-P_{\text{\emph{MD,M}}})^{y_i}.
\end{aligned}
$$ 
If the estimated robot identity vector $\hat{\boldsymbol{t}}$ is constructed by choosing $\hat{t_i}=1$ if $c_{\text{\emph{L}},i}\geq c_{\text{\emph{M}},i}$ and $\hat{t_i}=0$ otherwise, where $\hat{t_i}$ is the $i^{th}$ component of $\hat{\boldsymbol{t}}$, then, $\hat{\boldsymbol{t}}$ is a vector that maximizes the expression \eqref{eq:numerator_inner_max}.
\end{lemma}
\begin{proof}
First, we reformulate \eqref{eq:numerator_inner_max} as:
    \begin{flalign}\label{eq:numerator_inner_max_over_all_robots}
    &\prod_{i=1}^{N}(p_{\alpha}(a_i|t_i)P_{\text{MD,L}}^{1-y_i}(1-P_{\text{MD,L}})^{y_i})^{t_i} \cdot\nonumber\\ 
    &\qquad(p_{\alpha}(a_i|t_i)P_{\text{MD,M}}^{1-y_i}(1-P_{\text{MD,M}})^{y_i})^{1-t_i},
    \end{flalign}
where the product is calculated by going through all robots rather than going through legitimate and malicious robots separately. 
We define $$c_{\text{L},i}=p_{\alpha}(a_i|t_i)P_{\text{MD,L}}^{1-y_i}(1-P_{\text{MD,L}})^{y_i},$$ and $$c_{M,i}= p_{\alpha}(a_i|t_i)P_{\text{MD,M}}^{1-y_i}(1-P_{\text{MD,M}})^{y_i}.$$ Then, the expression in \eqref{eq:numerator_inner_max_over_all_robots} becomes:
    \begin{flalign}
    \label{eq:numerator_inner_max_over_all_robots_simplified}
    &\prod_{i=1}^{N}c_{\text{L},i}^{t_i} \cdot c_{\text{M},i}^{1-t_i}.
    \end{flalign}
Let $0\log{0}=1$, thus $0^0=1$. Then, the expression \eqref{eq:numerator_inner_max_over_all_robots_simplified} is maximized when choosing $t_i=1$ if $c_{L,i}\geq c_{M,i}$ and $t_i=0$ otherwise.
\end{proof}
As we can see from Lemma \ref{lem:lemma_P_given}, maximization with this formulation can be calculated by performing $\mathcal{O}(N)$ comparisons. Now, we consider these two perspectives together to introduce an efficient calculation of the numerator of the GLRT given in \eqref{eq:GLRT_numerator_maximization}. By Lemma \ref{lem:lemma_t_given}, we can see that the optimum value of $P_{\text{MD,M}}$ has a special structure. Exploiting this knowledge, we can restrict the set that $P_{\text{MD,M}}$ belongs to in \eqref{eq:GLRT_numerator_pma_first}. Then, the inner maximization can be calculated using Lemma \ref{lem:lemma_P_given}. The following theorem builds on this intuition to provide an efficient calculation of \eqref{eq:GLRT_numerator_maximization}.

\begin{theorem}
Assume that $(\boldsymbol{t}^*,P_{\text{MD,M}}^*)$ attains the maximization in \eqref{eq:GLRT_numerator_maximization}. Then, for each vector of measurements $\mathbf{y}$ and trust values $\mathbf{a}$, $P_{\text{MD,M}}^*$ belongs to the set $\mathcal{P}$ where
$$\mathcal{P}\triangleq\left\{\frac{T_n}{T_d}\right\}_{T_n\in\{0,\ldots,T_d\},T_d\in\{1,\ldots,N\}},$$
and $|\mathcal{P}|\leq N^2+1$. Moreover, the maximization in \eqref{eq:GLRT_numerator_maximization} can be calculated by iterating over $\mathcal{O}(N^2)$ different values in $\mathcal{P}$ and performing $\mathcal{O}(N)$ comparisons.
\label{thm:efficient_maximization}
\end{theorem}

\begin{proof}
First, we will approach the problem by rewriting it as \eqref{eq:GLRT_numerator_pma_first} using the Principle of Iterated Suprema:
$$
\begin{aligned}
\max_{P_{\text{MD,M}}\in[0,1]}\left\{\max_{\boldsymbol{t}\in\{0,1\}^N}\left\{\prod_{i:t_i=1}p_{\alpha}(a_i|t_i)P_{\text{MD,L}}^{1-y_i}(1-P_{\text{MD,L}})^{y_i} \right. \right. \cdot \\
\left. \left.\prod_{i:t_i=0}p_{\alpha}(a_i|t_i)P_{\text{MD,M}}^{1-y_i}(1-P_{\text{MD,M}})^{y_i}\right\}\right\},
\end{aligned}
$$
By Lemma \ref{lem:lemma_P_given}, we can calculate the inner maximization for a given $P_{\text{MD,M}}$. Notice that, since the calculation requires a comparison for each robot, $\mathcal{O}(N)$ comparisons need to be performed for this maximization. Now, consider the other formulation of the problem given by \eqref{eq:GLRT_numerator_t_first}. From Lemma \ref{lem:lemma_t_given}, we can see that the optimum  $P_{\text{MD,M}}$ only depends on the number of ones and zeros of malicious robots for a given $\boldsymbol{t}$. Moreover, the permutation of ones and zeros of malicious robots for a given $\boldsymbol{t}$ does not change the optimum and only the total number of ones and zeros does. We will restrict the set that the outer maximization process iterates over in \eqref{eq:GLRT_numerator_pma_first} based on this observation.
\\
Denote $$\mathcal{P}\triangleq\left\{\frac{T_n}{T_d}\right\}_{T_n\in\{0,\ldots,T_d\},T_d\in\{1,\ldots,N\}},$$
and observe that $|\mathcal{P}|\leq N^2+1$.
It follows from the Lemma \ref{lem:lemma_t_given} that for each value $\boldsymbol{t}$ in the outer maximization of \eqref{eq:GLRT_numerator_t_first}, except the case where $\boldsymbol{t}$ consist of all ones, the optimum value of $P_{\text{MD,M}}$ belongs to the set $\mathcal{P}$. Moreover, in the case where $\boldsymbol{t}$ consists of all ones, any choice of $P_{\text{MD,M}}$ maximizes the expression. Hence, without loss of generality, it is suffices to look for an optimizer $P_{\text{MD,M}}$ of \eqref{eq:GLRT_numerator_t_first} in the set $\mathcal{P}$.
Therefore, there are only  $\mathcal{O}(N^2)$ possible values that optimum $P_{\text{MD,M}}$ can take. Thus, we can reformulate \eqref{eq:GLRT_numerator_pma_first} as:
$$
\begin{aligned}
\max_{P_{\text{MD,M}}\in\mathcal{P}}\left\{\max_{\boldsymbol{t}\in\{0,1\}^N}\left\{\prod_{i:t_i=1}p_{\alpha}(a_i|t_i)P_{\text{MD,L}}^{1-{y_i}}(1-P_{\text{MD,L}})^{{y_i}} \right. \right. \cdot \\
\left. \left.\prod_{i:t_i=0}p_{\alpha}(a_i|t_i)P_{\text{MD,M}}^{1-{y_i}}(1-P_{\text{MD,M}})^{{y_i}}\right\}\right\},
\end{aligned}
$$
Therefore, this maximization can be calculated by iterating over $\mathcal{O}(N^2)$ different values of $P_{\text{MD},\text{M}}$ and for each value, performing $\mathcal{O}(N)$ comparisons. A similar approach can be adapted for calculating the denominator as well.
\end{proof}
Now, using Theorem \ref{thm:efficient_maximization}, we introduce the algorithm A-GLRT, which makes a decision based on the GLRT given by \eqref{eq:GLRT_ML_decision_independent}.\\
\begin{corollary}
The GLRT given by \eqref{eq:GLRT_ML_decision_independent} can be calculated by Algorithm \ref{alg:A-GLRT} which is referred as the A-GLRT algorithm. The A-GLRT algorithm requires $\mathcal{O}(N^3)$ comparisons.
\end{corollary}
\begin{proof}
Calculation of the maximization in the numerator can be calculated in $\mathcal{O}(N^2)$ iterations and performing $\mathcal{O}(N)$ comparisons at each iteration as described by Theorem \ref{thm:efficient_maximization}. Therefore, it requires $\mathcal{O}(N^3)$ comparisons in total. Similarly, maximization of the denominator requires the same amount of computation and can be calculated in a similar manner using $P_{\text{FA,M}}$ instead of $P_{\text{MD,M}}$. After that, a final comparison is made by comparing the ratio of the numerator and denominator with $\gamma_{\text{AG}}=\frac{\Pr(\Xi=0)}{\Pr(\Xi=1)}$. Algorithm \ref{alg:A-GLRT} follows these steps, therefore, it requires $\mathcal{O}(N^3)$ comparisons in total.
\end{proof}

\begin{algorithm}[h!]
\caption{A-GLRT \\ Input: $\mathbf{y}$, $\mathbf{a}$, $P_{\text{FA,L}}$, $P_{\text{MD,L}}$, $\Pr(\Xi=0)$, $\Pr(\Xi=1)$, $p_{\alpha}(a_i|t=1)$, $p_{\alpha}(a_i|t=0)$, N \\ Output: Decision $\mathcal{H}_0$ or $\mathcal{H}_1$}
\label{alg:A-GLRT}
\begin{algorithmic}[1]

\State Set $\mathcal{P} = \left\{\frac{T_n}{T_d}\right\}_{T_n\in\{0,\ldots,T_d\},T_d\in\{1,\ldots,N\}}$.

\State Set $\gamma_{\text{AG}} = \frac{\Pr(\Xi=0)}{\Pr(\Xi=1)}$.

\State Set $l_\text{num,max} = 0, l_\text{denom,max} = 0$.

\ForAll{$P_{M} \in \mathcal{P}$}

\State Set $P_{\text{MD,M}} = P_{M}, P_{\text{FA,M}} = P_{M}$.

\State Set $l_\text{num} = 1, l_\text{denom} = 1$.

\For{i=0 to N}

\State Set $c_{\text{L},i}=p_{\alpha}(a_i|t_i=1)P_{\text{MD,L}}^{1-y_i}(1-P_{\text{MD,L}})^{y_i}$.

\State Set $c_{\text{M},i}=p_{\alpha}(a_i|t_i=0)P_{\text{MD,M}}^{(1-y_i)}(1-P_{\text{MD},\text{M}})^{y_i}$.

\If {$c_{\text{L},i} \geq c_{\text{M},i}$}

\State Set $l_\text{num} = l_\text{num}\cdot c_{\text{L},i}$.

\Else

\State Set $l_\text{num} = l_\text{num}\cdot c_{\text{M},i}$.

\EndIf

\EndFor

\If {$l_\text{num} > l_\text{num,max}$}

\State Set $l_\text{num,max} = l_\text{num}$.

\EndIf

\State Repeat the steps 7-18 for the denominator.

\EndFor

\If {$\frac{l_\text{num,max}}{l_\text{denom,max}} > \gamma_{\text{AG}}$}
    \State Return decision $\mathcal{H}_1$ 
\Else
    \State Return decision $\mathcal{H}_0$ 
\EndIf

\end{algorithmic}
\end{algorithm}
Finally, we investigate how the measurements $\boldsymbol{y}$ and stochastic trust values $\boldsymbol{\alpha}$ are being used by the A-GLRT algorithm. Considering \eqref{eq:numerator_inner_max_over_all_robots}, an equivalent decision rule to the one derived in Lemma \ref{lem:lemma_P_given} is given as:
\begin{flalign}\label{eq:lemma_P_given_alternative_formulation_num}
\frac{p_{\alpha}(a_i|t_i=1)}{p_{\alpha}(a_i|t_i=0)}\underset{\hat{t}_i=0}{\overset{\hat{t}_i=1}{\begin{smallmatrix}\geqslant\\<\end{smallmatrix}}} \frac{P_{\text{MD,M}}^{1-y_i}(1-P_{\text{MD,M}})^{y_i}}{P_{\text{MD,L}}^{1-y_i}(1-P_{\text{MD,L}})^{y_i}}.
\end{flalign}

With this new perspective, we can gain more insights about the A-GLRT. First, we can see that the A-GLRT is essentially performing a likelihood ratio test with $\alpha$ values for each robot to decide if they are legitimate or not using different threshold values based on the measurement coming from that robot. For now, let's assume that $P_{\text{MD,M}}$ is not 0 or 1. Then, we can see that as $\alpha$ values become more accurate, meaning that the ratio $\frac{p_{\alpha}(a_i|t_i=1)}{p_{\alpha}(a_i|t_i=0)}$ approaches infinity if $t_i=1$ or approaches zero otherwise, for all values that $\alpha_i$ can take, the finite threshold value becomes insignificant and the decision is made using $\alpha$ values only. This situation agrees with the intuition as well since $\alpha$ values would become true indicators of robot identities.

\section{Hardware Experiment and Numerical Results}
\label{sec:Results}
We perform a hardware experiment with robotic vehicles driving on a mock-up road network where robots are tasked with reporting the traffic condition of their road segment to a FC. The objective of the malicious robots is to cause the FC to incorrectly perceive the traffic conditions (see Fig.~\ref{fig:hardware_setup}). A numerical study further demonstrates the performance of this scenario with an increasing proportion of malicious robots.

We compare the performance of the 2SA and A-GLRT against several benchmarks including the \emph{\textbf{Oracle}}, where the FC knows the true trust vector $\mathbf{t}$ and discards malicious measurements, (this serves as a lower bound on the probability of error), the \emph{\textbf{Oblivious FC}}, where the FC treats every robot as legitimate, and a \emph{\textbf{Baseline Approach}} \cite{rawat2010collaborative} where the FC uses a history of $T$ measurements to develop a reputation about each robot. The Baseline method ignores information from robots whose measurements disagree with the final decision at least $\eta < T$ times. The \emph{Oracle}, \emph{Oblivious FC}, and \emph{Baseline Approach} use the decision rule in \eqref{eq:detection_p_e_FC_legitimate_assumption}. Malicious robots perform a Sybil attack where they spoof additional robots into the network. We use the opensource toolbox in \cite{WSR_toolbox} to obtain trust values from communicated WiFi signals by analyzing the similarity between different fingerprints to detect spoofed transmissions. The works in \cite{yemini2021characterizing,AURO,CrowdVetting} model these trust values $\alpha_i \in [0,1]$ as a continuous random variable. We discretize the sample space by letting $\mathcal{A} = \{0,1\}$ and setting $a_i = 1$ if the measured trust value is $\geq 0.5$ and $a_i = 0$ otherwise.

\paragraph{\textbf{Hardware Experiment}}

\begin{figure}[t!]
    \centering
    \includegraphics[scale=0.35]{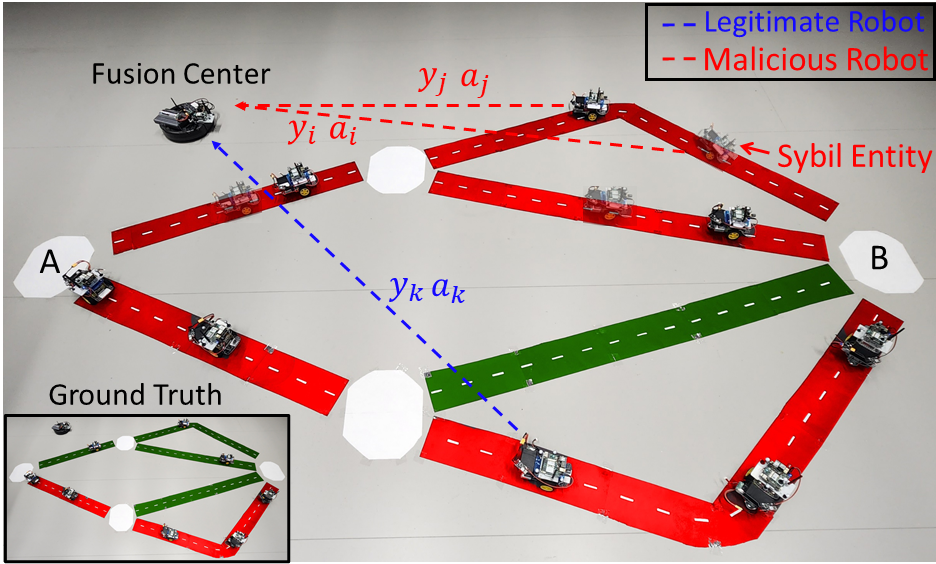}
    \caption{Robots drive along a roadmap comprised of six road segments to get from point A to point B. While traversing the roadmap, robots estimate the congestion on their current road segment as either containing traffic (red) or not (green), and relay their estimates to the FC. All robots relay messages to the FC, but only a few are depicted on the figure for ease of readability.}
    \label{fig:hardware_setup}
\end{figure}

A group of $N = 11$ mobile robots drive in a loop from a starting point A to point B, approximately $4.5$ meters apart, by traversing one of four possible paths made up of six different road segments. As the robots drive between points A and B they are given noisy position information for themselves and neighboring robots from an OptiTrack motion capture system with added white Gaussian noise with a variance of $1m^2$. This serves as a proxy for GPS-reported measures used in crowdsourcing traffic estimation schemes like Waze, Google Maps, and others. A road segment is considered to have traffic ($y_i = 1$) if the number of robots on the segment is $\geq 2$. Of the $11$ robots in the group, $5$ robots are legitimate, $3$ are malicious, and $3$ are spoofed by the malicious robots (making them also malicious). Malicious robots know the true traffic conditions and report the wrong measurement with probability $0.99$, i.e., $P_{\text{FA,M}} = P_{\text{MD,M}} =0.99$. The empirical data from the experiment is stated in \cref{tab:experiment}, where \emph{Baseline1} and \emph{Baseline5} refer to the Baseline Approach from \cite{rawat2010collaborative} with parameters $T$ and $\eta$ set to ($T = 1$, $\eta = 0.5$) and ($T = 5$, $\eta = 2.5$). We determined the parameters in \cref{tab:experiment} by first running an experiment without performing hypothesis tests and observing the behavior of the system compared to ground truth. The trust values gathered using the toolbox in \cite{WSR_toolbox} led to the empirical probabilities $p_{\alpha}(a_i=1|t_i=1) = 0.8350$ and $p_{\alpha}(a_i=1|t_i=0) = 0.1691$ (see \cref{fig:histograms}).

In our hardware experiment the 2SA and A-GLRT outperform the Oblivious FC and the Baseline Approach. The Baseline Approach exhibits a high percent error due to the fact that it relies on the majority of the network being legitimate. Since $6$ out of $11$ robots are malicious, it is likely that many hypothesis tests are conducted where the majority is malicious. This points to a common vulnerability of reputation based approaches that assume only a small proportion of the network is malicious.

\begin{table}[t!]
    \centering
    \begin{tabular}{|c|c|c|c|} \hline
    \multicolumn{4}{|c|}{\textbf{\textit{Parameters}}} \\ \hline
    $P_{\text{FA,L}}$  & 0.0800 & $P_{\text{MD,L}}$ & 0.2100 \\ \hline  $\Pr(\Xi=0)$  & 0.6432 & $\Pr(\Xi=1)$ & 0.3568 \\ \hline \multicolumn{4}{|c|}{\textbf{\textit{Percent Error}}} \\ \hline \textbf{2SA (Sec. \ref{sec:2SA})} & \textbf{30.5} \% &\textbf{A-GLRT (Sec. \ref{sec:GLRT})} & \textbf{29.0} \% \\ \hline Oracle & 19.5 \% & Oblivious FC & 52.0 \% \\ \hline Baseline1 & 50.8 \% & Baseline5 & 49.1 \% \\ \hline
    \end{tabular}
    \caption{Experimental Results}
    \label{tab:experiment}
\end{table}
\begin{figure}[t!]
    \centering
    \includegraphics[scale=0.28]{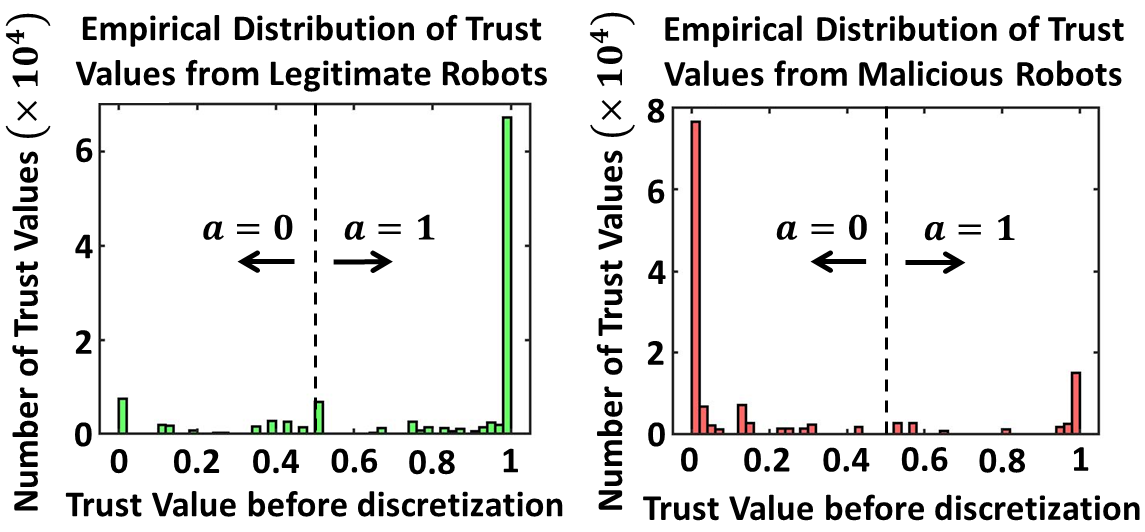}
    \caption{Empirical distribution of the trust values gathered during the hardware experiment for legitimate and malicious robots. The trust value is thresholded to $a=1$ if it is $\geq 0.5$, and $a=0$ otherwise.}
    \label{fig:histograms}
\end{figure}

\paragraph{\textbf{Numerical Study}}

Next, we perform a numerical study on the performance of each approach when the proportion of malicious robots is varied. In the numerical study we use $N = 10$ robots with $\Pr(\Xi = 0) = \Pr(\Xi = 1) = 0.5$, $P_{\text{FA,L}} = P_{\text{MD,L}} = 0.15$, and $P_{\text{FA,M}} = P_{\text{MD,M}} = 0.99$ and perform hypothesis tests over $1000$ trials for each proportion of malicious robots. In the simulation study the trust value distributions are fixed at $p_{\alpha}(a_i = 1|t_i = 1) = 0.8$, $p_{\alpha}(a_i = 1|t_i = 0) = 0.2$, and the proportion of malicious robots varies from $0$ to $1$. The results of the simulation study are plotted in \cref{fig:sim_study}. From the plot it can be seen that the 2SA and the A-GLRT perform well even after the number of malicious robots exceeds majority since they use additional trust information independent of the data, whereas the Baseline Approaches fail since they use only the data to assess the trustworthiness of the robots. Additionally, there exists a critical proportion of malicious robots, 
beyond which
the 2SA chooses to ignore most of the measurements it receives and the decision rule becomes more dependent on the prior probabilities $\Pr(\Xi=0)$ and $\Pr(\Xi=1)$.
\begin{figure}[t!]
    \centering
    \includegraphics[scale=0.6]{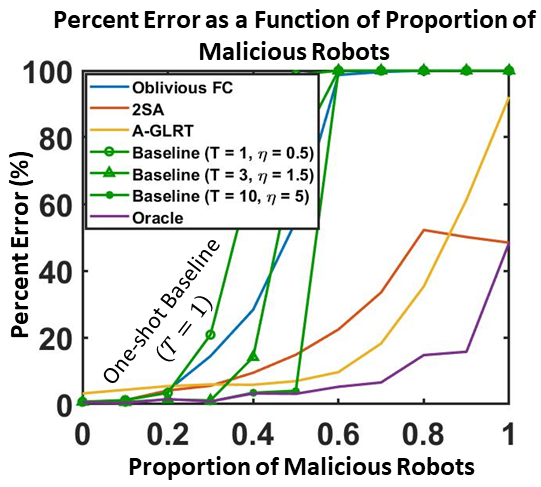}
    \caption{The percent error for multiple hypothesis test approaches when the proportion of malicious robots is varied. The 2SA and A-GLRT outperform the Oblivious FC and Baseline Approaches when the majority of the network is malicious.}
    \label{fig:sim_study}
\end{figure}

\addtolength{\textheight}{-4cm}   

\section{Conclusion}
\label{sec:conclusion}

In this paper we present two methods to utilize trust values in solving the binary adversarial hypothesis testing problem. The 2SA uses the trust values to determine which robots to trust, and then makes a decision from the measurements of the trusted robots. The A-GLRT jointly uses the trust values and measurements to estimate the trustworthiness of each robot, the strategy of malicious robots, and the true hypothesis.

\appendix
\subsection{Proof of \cref{lem:P_FA_M}} \label{sec:proof_P_FA_M}

\begin{proof}
 Recall the false alarm and missed detection probabilities for the FC using decision rules \eqref{eq:classification_decision} and \eqref{eq:detection_p_e_FC_legitimate_assumption} that lead to the overall false alarm and missed detection probabilities stated in \eqref{eq:fa_md_total}.

Next, we show that the false alarm probability \eqref{eq:fa_md_total} is maximized when $P_{\text{FA,M}} = 1$. The proof for $P_{\text{MD,M}}$ is analogous. In order to maximize $P_{\text{FA}}$ in \eqref{eq:fa_md_total} the summation must be maximized. We rewrite the summation by separating it into the terms affected by legitimate robots that were trusted and those affected by malicious robots that were trusted
\begin{equation}
\begin{aligned}
    &\sum_{i : \{ \hat{t}_i=1, t_i=0 \}}[w_{1,\text{L}} y_j - w_{0,\text{L}}(1-y_j)] + \\ &\sum_{i : \{ \hat{t}_i=1, t_i=1 \}} [w_{1,\text{L}} y_i - w_{0,\text{L}}(1-y_i)].
    \label{eq:lem_proof_P_FA_M_summation}
\end{aligned}
\end{equation}
Any robot $j \in \{ \hat{\mathcal{L}}\cap\mathcal{M} \}$ can maximize \eqref{eq:lem_proof_P_FA_M_summation} by maximizing $[w_{1,\text{L}} y_j - w_{0,\text{L}}(1-y_j)]$. Note that when $P_{\text{FA,L}} < 0.5$ and $P_{\text{MD,L}} < 0.5$ then $w_{1,\text{L}} > 0$ and $w_{0,\text{L}} > 0$. Thus, $[w_{1,\text{L}} y_j - w_{0,\text{L}}(1-y_j)]$ is maximized when $y_j = 1$ since $Y_j \in \{0,1\}$. Given the true hypothesis is $\mathcal{H}_0$, the measurement $Y_j = 1$ occurs when robot $j$ reports a false alarm. Therefore, the probability that robot $j$ reports $Y_j = 1$ is maximized when the probability of false alarm is maximized:
\begin{equation}
    \Pr(Y_j = 1 | \mathcal{H}_0) = P_{\text{FA,M}} = 1.
\end{equation}
\end{proof}


\bibliographystyle{IEEEtran}

\begin{thebibliography}{10}
\providecommand{\url}[1]{#1}
\csname url@samestyle\endcsname
\providecommand{\newblock}{\relax}
\providecommand{\bibinfo}[2]{#2}
\providecommand{\BIBentrySTDinterwordspacing}{\spaceskip=0pt\relax}
\providecommand{\BIBentryALTinterwordstretchfactor}{4}
\providecommand{\BIBentryALTinterwordspacing}{\spaceskip=\fontdimen2\font plus
\BIBentryALTinterwordstretchfactor\fontdimen3\font minus
  \fontdimen4\font\relax}
\providecommand{\BIBforeignlanguage}[2]{{%
\expandafter\ifx\csname l@#1\endcsname\relax
\typeout{** WARNING: IEEEtran.bst: No hyphenation pattern has been}%
\typeout{** loaded for the language `#1'. Using the pattern for}%
\typeout{** the default language instead.}%
\else
\language=\csname l@#1\endcsname
\fi
#2}}
\providecommand{\BIBdecl}{\relax}
\BIBdecl

\bibitem{kailkhura2014asymptotic}
B.~Kailkhura, Y.~S. Han, S.~Brahma, and P.~K. Varshney, ``Asymptotic analysis
  of distributed bayesian detection with byzantine data,'' \emph{IEEE Signal
  Processing Letters}, vol.~22, no.~5, pp. 608--612, 2014.

\bibitem{ren2018binary}
X.~Ren, J.~Yan, and Y.~Mo, ``Binary hypothesis testing with byzantine sensors:
  Fundamental tradeoff between security and efficiency,'' \emph{IEEE
  Transactions on Signal Processing}, vol.~66, no.~6, pp. 1454--1468, 2018.

\bibitem{althunibat2016countering}
S.~Althunibat, A.~Antonopoulos, E.~Kartsakli, F.~Granelli, and C.~Verikoukis,
  ``Countering intelligent-dependent malicious nodes in target detection
  wireless sensor networks,'' \emph{IEEE Sensors Journal}, vol.~16, no.~23, pp.
  8627--8639, 2016.

\bibitem{wu2018generalized}
J.~Wu, T.~Song, Y.~Yu, C.~Wang, and J.~Hu, ``Generalized byzantine attack and
  defense in cooperative spectrum sensing for cognitive radio networks,''
  \emph{IEEE Access}, vol.~6, pp. 53\,272--53\,286, 2018.

\bibitem{robotTrustSchwager}
A.~Pierson and M.~Schwager, ``Adaptive inter-robot trust for robust multi-robot
  sensor coverage,'' in \emph{In International Symposium on Robotics Research},
  2013.

\bibitem{xu2021novel}
Y.~Xu, G.~Deng, T.~Zhang, H.~Qiu, and Y.~Bao, ``Novel denial-of-service attacks
  against cloud-based multi-robot systems,'' \emph{Information Sciences}, vol.
  576, pp. 329--344, 2021.

\bibitem{song2020care}
J.~Song and S.~Gupta, ``Care: Cooperative autonomy for resilience and
  efficiency of robot teams for complete coverage of unknown environments under
  robot failures,'' \emph{Autonomous Robots}, vol.~44, no.~3, pp. 647--671,
  2020.

\bibitem{talay2009task}
S.~Sariel-Talay, T.~R. Balch, and N.~Erdogan, ``Multiple traveling robot
  problem: A solution based on dynamic task selection and robust execution,''
  \emph{IEEE/ASME TRANSACTIONS ON MECHATRONICS}, vol.~14, no.~2, 2009.

\bibitem{schlotfeldt2018resilient}
B.~Schlotfeldt, V.~Tzoumas, D.~Thakur, and G.~J. Pappas, ``Resilient active
  information gathering with mobile robots,'' in \emph{2018 IEEE/RSJ
  International Conference on Intelligent Robots and Systems (IROS)}.\hskip 1em
  plus 0.5em minus 0.4em\relax IEEE, 2018, pp. 4309--4316.

\bibitem{ramachandran2020resilience}
R.~K. Ramachandran, N.~Fronda, and G.~S. Sukhatme, ``Resilience in multi-robot
  target tracking through reconfiguration,'' in \emph{2020 IEEE International
  Conference on Robotics and Automation (ICRA)}.\hskip 1em plus 0.5em minus
  0.4em\relax IEEE, 2020, pp. 4551--4557.

\bibitem{mitra2019resilient}
A.~Mitra, J.~A. Richards, S.~Bagchi, and S.~Sundaram, ``Resilient distributed
  state estimation with mobile agents: overcoming byzantine adversaries,
  communication losses, and intermittent measurements,'' \emph{Autonomous
  Robots}, vol.~43, no.~3, pp. 743--768, 2019.

\bibitem{laszka2015resilient}
A.~Laszka, Y.~Vorobeychik, and X.~Koutsoukos, ``Resilient observation selection
  in adversarial settings,'' in \emph{2015 54th IEEE Conference on Decision and
  Control (CDC)}.\hskip 1em plus 0.5em minus 0.4em\relax IEEE, 2015, pp.
  7416--7421.

\bibitem{blumenkamp2021emergence}
J.~Blumenkamp and A.~Prorok, ``The emergence of adversarial communication in
  multi-agent reinforcement learning,'' in \emph{Conference on Robot
  Learning}.\hskip 1em plus 0.5em minus 0.4em\relax PMLR, 2021, pp. 1394--1414.

\bibitem{mitchell2020gaussian}
R.~Mitchell, J.~Blumenkamp, and A.~Prorok, ``Gaussian process based message
  filtering for robust multi-agent cooperation in the presence of adversarial
  communication,'' \emph{arXiv preprint arXiv:2012.00508}, 2020.

\bibitem{deng2021byz}
G.~Deng, Y.~Zhou, Y.~Xu, T.~Zhang, and Y.~Liu, ``An investigation of byzantine
  threats in multi-robot systems,'' in \emph{24th International Symposium on
  Research in Attacks, Intrusions and Defenses}, 2021, pp. 17--32.

\bibitem{wehbe2022probabilistically}
R.~Wehbe and R.~K. Williams, ``Probabilistically resilient multi-robot
  informative path planning,'' \emph{arXiv preprint arXiv:2206.11789}, 2022.

\bibitem{GoogleMaps}
N.~Petrovska and A.~Stevanovic, ``Traffic congestion analysis visualisation
  tool,'' in \emph{2015 IEEE 18th International Conference on Intelligent
  Transportation Systems}.\hskip 1em plus 0.5em minus 0.4em\relax IEEE, 2015,
  pp. 1489--1494.

\bibitem{jeske2013floating}
T.~Jeske, ``Floating car data from smartphones: What google and waze know about
  you and how hackers can control traffic,'' \emph{Proc. of the BlackHat
  Europe}, pp. 1--12, 2013.

\bibitem{wang2018ghost}
G.~Wang, B.~Wang, T.~Wang, A.~Nika, H.~Zheng, and B.~Y. Zhao, ``Ghost riders:
  Sybil attacks on crowdsourced mobile mapping services,'' \emph{IEEE/ACM
  transactions on networking}, vol.~26, no.~3, pp. 1123--1136, 2018.

\bibitem{sandal2020reputation}
Y.~S. Sandal, A.~E. Pusane, G.~K. Kurt, and F.~Benedetto, ``Reputation based
  attacker identification policy for multi-access edge computing in internet of
  things,'' \emph{IEEE Transactions on Vehicular Technology}, vol.~69, no.~12,
  pp. 15\,346--15\,356, 2020.

\bibitem{marano2008distributed}
S.~Marano, V.~Matta, and L.~Tong, ``Distributed detection in the presence of
  byzantine attacks,'' \emph{IEEE Transactions on Signal Processing}, vol.~57,
  no.~1, pp. 16--29, 2008.

\bibitem{kailkhura2015distributed}
B.~Kailkhura, Y.~S. Han, S.~Brahma, and P.~K. Varshney, ``Distributed bayesian
  detection in the presence of byzantine data,'' \emph{IEEE transactions on
  signal processing}, vol.~63, no.~19, pp. 5250--5263, 2015.

\bibitem{chen2008robust}
R.~Chen, J.-M. Park, and K.~Bian, ``Robust distributed spectrum sensing in
  cognitive radio networks,'' in \emph{IEEE INFOCOM 2008-The 27th Conference on
  Computer Communications}.\hskip 1em plus 0.5em minus 0.4em\relax IEEE, 2008,
  pp. 1876--1884.

\bibitem{nurellari2017secure}
E.~Nurellari, D.~McLernon, and M.~Ghogho, ``A secure optimum distributed
  detection scheme in under-attack wireless sensor networks,'' \emph{IEEE
  Transactions on Signal and Information Processing over Networks}, vol.~4,
  no.~2, pp. 325--337, 2017.

\bibitem{nurellari2016distributed}
E.~Nurellari, D.~McLernon, M.~Ghogho, and S.~Aldalahmeh, ``Distributed binary
  event detection under data-falsification and energy-bandwidth limitation,''
  \emph{IEEE Sensors Journal}, vol.~16, no.~16, pp. 6298--6309, 2016.

\bibitem{rawat2010collaborative}
A.~S. Rawat, P.~Anand, H.~Chen, and P.~K. Varshney, ``Collaborative spectrum
  sensing in the presence of byzantine attacks in cognitive radio networks,''
  \emph{IEEE Transactions on Signal Processing}, vol.~59, no.~2, pp. 774--786,
  2010.

\bibitem{trustandRobotsSycara}
R.~Liu, F.~Jia, W.~Luo, M.~Chandarana, C.~Nam, M.~Lewis, and K.~Sycara,
  ``Trust-aware behavior reflection for robot swarm self-healing,''
  \emph{Proceedings of the 18th International Conference on Autonomous Agents
  and MultiAgent Systems}, p. 122–130, 2019.

\bibitem{spoofResilientCoordinationusingFingerprints}
V.~{Renganathan} and T.~{Summers}, ``Spoof resilient coordination for
  distributed multi-robot systems,'' \emph{2017 International Symposium on
  Multi-Robot and Multi-Agent Systems (MRS)}, pp. 135--141, Dec 2017.

\bibitem{securearray}
J.~Xiong and K.~Jamieson, ``Securearray: Improving wifi security with
  fine-grained physical-layer information,'' \emph{Proceedings of the 19th
  Annual International Conference on Mobile Computing \& Networking}, p.
  441–452, 2013.

\bibitem{AURO}
S.~Gil, S.~Kumar, M.~Mazumder, D.~Katabi, and D.~Rus, ``Guaranteeing
  spoof-resilient multi-robot networks,'' \emph{AuRo}, p. 1383–1400, 2017.

\bibitem{CrowdVetting}
F.~Mallmann-Trenn, M.~Cavorsi, and S.~Gil, ``Crowd vetting: Rejecting
  adversaries via collaboration with application to multirobot flocking,''
  \emph{IEEE Transactions on Robotics}, vol.~38, no.~1, pp. 5--24, 2022.

\bibitem{yemini2021characterizing}
M.~Yemini, A.~Nedi{\'c}, A.~J. Goldsmith, and S.~Gil, ``Characterizing trust
  and resilience in distributed consensus for cyberphysical systems,''
  \emph{IEEE Transactions on Robotics}, vol.~38, no.~1, pp. 71--91, 2021.

\bibitem{soltanmohammadi2012decentralized}
E.~Soltanmohammadi, M.~Orooji, and M.~Naraghi-Pour, ``Decentralized hypothesis
  testing in wireless sensor networks in the presence of misbehaving nodes,''
  \emph{IEEE Transactions on Information Forensics and Security}, vol.~8,
  no.~1, pp. 205--215, 2012.

\bibitem{kay_2008}
S.~M. Kay, \emph{Fundamentals of statistical signal processing: Detection
  theory}.\hskip 1em plus 0.5em minus 0.4em\relax Prentice Hall PTR, 2008.

\bibitem{sun2016optimal}
Z.~Sun, C.~Zhang, and P.~Fan, ``Optimal byzantine attack and byzantine
  identification in distributed sensor networks,'' in \emph{2016 IEEE Globecom
  Workshops (GC Wkshps)}.\hskip 1em plus 0.5em minus 0.4em\relax IEEE, 2016,
  pp. 1--6.

\bibitem{cheng2021general}
M.~Cheng, C.~Yin, J.~Zhang, S.~Nazarian, J.~Deshmukh, and P.~Bogdan, ``A
  general trust framework for multi-agent systems,'' in \emph{Proceedings of
  the 20th International Conference on Autonomous Agents and MultiAgent
  Systems}, 2021, pp. 332--340.

\bibitem{peng2012agenttms}
M.~Peng, Z.~Xu, S.~Pan, R.~Li, and T.~Mao, ``Agenttms: A mas trust model based
  on agent social relationship.'' \emph{J. Comput.}, vol.~7, no.~6, pp.
  1535--1542, 2012.

\bibitem{olmsted1959real}
J.~M.~H. Olmsted, \emph{Real variables: An introduction to the theory of
  functions}.\hskip 1em plus 0.5em minus 0.4em\relax Appleton-Century-Crofts,
  1959.

\bibitem{kay1993estimation}
S.~Kay, \emph{Fundamentals of Statistical Signal Processing, Volume {I}:
  Estimation Theory}.\hskip 1em plus 0.5em minus 0.4em\relax Prentice-Hall PTR,
  1993.

\bibitem{WSR_toolbox}
N.~Jadhav, W.~Wang, D.~Zhang, S.~Kumar, and S.~Gil, ``Toolbox release: A
  wifi-based relative bearing sensor for robotics,'' \emph{ArXiv}, vol.
  abs/2109.12205, 2021.

\end{thebibliography}

\end{document}